%% file: dp_lm.tex
\documentclass{article}
\usepackage{iclr2018_conference,times}

\usepackage{algorithm}
\usepackage[noend]{algorithmic}

\input{includes}
\iclrfinalcopy

\title{Learning Differentially Private Recurrent Language Models}
\author{
 H. Brendan McMahan \\
 \texttt{mcmahan@google.com} \\
 \And
 Daniel Ramage\\
 \texttt{dramage@google.com}\\
 \And
 Kunal Talwar\\
 \texttt{kunal@google.com}\\
 \And
 Li Zhang\\
 \texttt{liqzhang@google.com}\\
}

\begin{document}

\maketitle

\input{intro}

\input{defn}

\input{algorithm}

\input{experiments}

\section{Conclusions}
In this work, we introduced an algorithm for user-level differentially
private training of large neural networks, in particular a complex
sequence model for next-word prediction.  We empirically evaluated the
algorithm on a realistic dataset and demonstrated that such training
is possible at a negligible loss in utility, instead paying a cost in
additional computation.
Such private training, combined with federated learning (which leaves
the sensitive training data on device rather than centralizing it),
shows the possibility of training models with significant privacy
guarantees for important real world applications.
Much future work remains, for example designing private algorithms
that automate and make adaptive the tuning of the clipping/noise
tradeoff, and the application to a wider range of model families and
architectures, for example GRUs and character-level models. Our work
also highlights the open direction of reducing the computational overhead
of differentially private training of non-convex models.

\begin{small}
\bibliography{dp_lm.bib}
\end{small}

\clearpage

\appendix
\input{algorithm-supp}

\input{experiments_supplement}

\end{document}

%% file: includes.tex
\usepackage{hyperref}
\usepackage{booktabs}       %
\usepackage{amsfonts}       %
\usepackage{times}
\usepackage{subcaption}

\usepackage[shortlabels]{enumitem}
\setlist{nolistsep}
\setlist[itemize]{noitemsep, topsep=0pt}

\usepackage{graphicx} %
\usepackage{xcolor}

\usepackage{array}
\usepackage{amssymb}
\usepackage{amsmath}
\usepackage{amsthm}
\usepackage{xspace}

\newtheorem{theorem}{Theorem}
\newtheorem{lemma}{Lemma}

\newtheorem{definition}{Definition}

\newcommand{\mycaptionof}[2]{\captionof{#1}{#2}}

\newcommand{\secref}[1]{$\S$\ref{sec:#1}}

\newcommand{\maxuserweight}{\hat{w}}
\newcommand{\update}{\Delta}
\newcommand{\usersperround}{C}
\newcommand{\eC}{\tilde{C}} %

\newcommand{\numusers}{K}
\newcommand{\numlayers}{m}
\newcommand{\moments}{\mathcal{M}}
\newcommand{\noisescale}{z}
\newcommand{\sensitivity}{\mathbb{S}}

\newcommand{\sample}{\mathcal{C}}

\newcommand{\Wmin}{W_{\text{min}}}
\newcommand{\fixedest}{\tilde{f}_\mathrm{f}} %
\newcommand{\varyest}{\tilde{f}_\mathrm{c}} %
\newcommand{\clipfn}{\text{ClipFn}}

\renewcommand{\mp}{\theta} %
\newcommand{\ltwo}{L_2}

\newcommand{\Num}{\operatorname{Num}}
\newcommand{\Den}{\operatorname{Den}}

\newcolumntype{H}{>{\setbox0=\hbox\bgroup}c<{\egroup}@{}}

\definecolor{darkgreen}{rgb}{0,0.4,0.0}
\definecolor{darkblue}{rgb}{0,0.1,0.3}
\definecolor{darkred}{rgb}{0.7,0.0,0.0}

\definecolor{modelgray}{HTML}{808080}
\definecolor{modelblue}{rgb}{0.2980392156862745, 0.4470588235294118, 0.6901960784313725}
\definecolor{modelred}{rgb}{0.7686274509803922, 0.3058823529411765, 0.3215686274509804}

\newcommand{\noiclr}[1]{}

\newcommand{\mech}{\mathcal{M}}
\newcommand{\database}{d}  %
\newcommand{\privdomain}{\mathcal{D}}
\newcommand{\privrange}{\mathcal{R}}

\newcommand{\clip}[2]{\operatorname{\pi}(#1, #2)}
\newcommand {\norm}[1]{\ensuremath{\| #1 \|}}

\newcommand{\lbs}{\ensuremath{B}}      %

\newcommand{\lepochs}{\ensuremath{E}}  %

\newcommand{\loss}{\ell}
\newcommand{\SUB}[1]{\ENSURE \hspace{-0.15in} \textbf{#1}}

\newcommand{\algfont}[1]{\texttt{#1}}

\newcommand{\fedavglong}{\algfont{FederatedAveraging}\xspace}
\newcommand{\fedavg}{\algfont{FedAvg}\xspace}

\newcommand{\dpfedavg}{\algfont{DP-FedAvg}\xspace}

\newcommand{\fedsgdlong}{\algfont{FederatedSGD}\xspace}
\newcommand{\fedsgd}{\algfont{FedSGD}\xspace}

\newcommand{\dpfedsgd}{\algfont{DP-FedSGD}\xspace}

\newcommand{\eqdef}{\overset{\text{def}}{=}}
\newcommand{\T}{\rule{0pt}{2.2ex}}

\newlength{\pw}

\newcommand{\R}{\ensuremath{\mathbb{R}}}

\DeclareMathOperator*{\E}{\mathbb{E}}
\newcommand{\grad}{\triangledown}

\newcommand{\eps}{\epsilon}
\newcommand{\Normal}{\mathcal{N}}

%% file: intro.tex
\begin{abstract}
We demonstrate that it is possible to train large recurrent language models with user-level differential privacy guarantees with only a negligible cost in predictive accuracy.  Our work builds on recent advances in the training of deep networks on user-partitioned data and privacy accounting for stochastic gradient descent. In particular, we add user-level privacy protection to the federated averaging algorithm, which makes ``large step'' updates from user-level data. Our work demonstrates that given a dataset with a sufficiently large number of users (a requirement easily met by even small internet-scale datasets), achieving differential privacy comes at the cost of increased computation, rather than in decreased utility as in most prior work. We find that our private LSTM language models are quantitatively and qualitatively similar to un-noised models when trained on a large dataset.
\end{abstract}

\section{Introduction}

Deep recurrent models like long short-term memory (LSTM) recurrent
neural networks (RNNs) have become a standard building block in modern
approaches to language modeling, with applications in speech
recognition, input decoding for mobile keyboards, and language
translation. Because language usage varies widely by problem domain
and dataset, training a language model on data from the right
distribution is critical. For example, a model to aid typing on a
mobile keyboard is better served by training data typed in mobile apps
rather than from scanned books or transcribed utterances.
However, language data can be uniquely privacy sensitive. In the case
of text typed on a mobile phone, this sensitive information might
include passwords, text messages, and search queries. In general,
language data may identify a speaker---explicitly by name or
implicitly, for example via a rare or unique phrase---and link that
speaker to secret or sensitive information.  

Ideally, a language model's parameters would encode patterns of
language use common to many users without memorizing any individual
user's unique input sequences. However, we know convolutional NNs can
memorize arbitrary labelings of the training
data~\citep{zhang17understanding} and  recurrent language
models are also capable of memorizing unique patterns in the
training data~\citep{carlini18secretsharer}.  Recent attacks on neural networks such as those of
\citet{ShokriSSS17} underscore the implicit risk. The main goal of
our work is to provide a strong guarantee that the trained model
protects the privacy of individuals' data without undue sacrifice in
model quality.
\noiclr{  building upon recent work in applying differential
privacy to machine learning. (e.g.,
  \citep{dwork14book,DworkLei09,KasiviswanathanLNRS11,chaudhuri11dperm,
    BeimelBKN14,bassily14focs,McSherryMironov09,SongCS13,
    ShokriShmatikov15, abadi16dpdl,PapernotAEGT17,WLKCJN17}.}

We are motivated by the problem of training models for next-word
prediction in a mobile keyboard, and use this as a running
example. This problem is well suited to the techniques we introduce,
as differential privacy may allow for training on data from the true
distribution (actual mobile usage) rather than on proxy data from some
other source that would produce inferior models. However, to
facilitate reproducibility and comparison to non-private models, our
experiments are conducted on a public dataset as is standard in
differential privacy research.
The remainder of this paper is structured around the following
contributions:

1. We apply differential privacy to model training using the notion of
\textit{user-adjacent} datasets, leading to formal guarantees of
user-level privacy, rather than privacy for single examples.

2. We introduce a noised version of the federated averaging algorithm
\citep{mcmahan16fedavg} in \secref{mechanism}, which satisfies
user-adjacent differential privacy via use of the moments accountant
\citep{abadi16dpdl} first developed to analyze differentially private
stochastic gradient descent (SGD) for example-level privacy. The
federated averaging approach groups multiple SGD updates together,
enabling large-step model updates.

  3. We demonstrate the first high quality LSTM language model trained
  with strong privacy guarantees in \secref{experiments}, showing no
  significant decrease in model accuracy given a large enough dataset.
  For example, on a dataset of 763,430 users, baseline (non-private)
  training achieves an accuracy of $17.5\%$ in 4120 rounds of
  training, where we use the data from 100 random users on each
  round. We achieve this same level of accuracy with $(4.6,
  10^{-9})$-differential privacy in 4980 rounds, processing on average
  5000 users per round---maintaining the same level of accuracy at a
  significant computational cost of roughly $60\times$.\footnote{The
    additional computational cost could be mitigated by initializing
    by training on a public dataset, rather than starting from random
    initialization as we do in our experiments.} Running the same
  computation on a larger dataset with $10^8$ users would improve the
  privacy guarantee to $(1.2, 10^{-9})$.
  We guarantee privacy and maintain utility despite the complex
  internal structure of the LSTM---with per-word embeddings as well as
  dense state transitions---by using the federated averaging
  algorithm.  We demonstrate that the noised model's metrics and
  qualitative behavior (with respect to head words) does not differ
  significantly from the non-private model.  To our knowledge, our
  work represents the most sophisticated machine learning model,
  judged by the size and the complexity of the model, ever trained
  with privacy guarantees, and the first such model trained with
  user-level privacy.

  4. In extensive experiments in \secref{experiments}, we
  offer guidelines for parameter tuning when training complex models with
  differential privacy guarantees. We show that a small number of
  experiments can narrow the parameter space into a regime where we
  pay for privacy not in terms of a loss in utility but in terms of an
  increased computational cost.

%% file: defn.tex
We now introduce a few preliminaries. Differential
privacy (DP)~\citep{dwork06calibrating, dwork11founation, dwork14book}
provides a well-tested formalization for the release of information
derived from private data.  Applied to machine learning, a
differentially private training mechanism allows the public release of
model parameters with a strong guarantee: adversaries are severely
limited in what they can learn about the original training data based
on analyzing the parameters, even when they have access to
arbitrary side information. Formally, it says:
\begin{definition}Differential Privacy:
  A randomized mechanism $\mech \colon \privdomain
  \rightarrow \privrange$ with a domain $\privdomain$ (e.g., possible
  training datasets) and range $\privrange$ (e.g., all possible
  trained models) satisfies $(\eps,\delta)$-differential privacy if
  for any two \textbf{adjacent} datasets $\database, \database'\in \privdomain$
  and for any subset of outputs $S \subseteq \privrange$ it holds that
  $
    \Pr[\mech(\database)\in S]\leq e^{\eps}\Pr[\mech(\database') \in S] + \delta.
  $
\end{definition}

The definition above leaves open the definition of \textbf{adjacent
  datasets} which will depend on the application. Most prior work on
differentially private machine learning
(e.g.~\citet{chaudhuri11dperm,bassily14focs,abadi16dpdl, WLKCJN17,
  PapernotAEGT17}) deals with {\em example-level privacy}: two
datasets $\database$ and $\database'$ are defined to be adjacent if
$\database'$ can be formed by adding or removing a single training
example from $\database$. We remark that while the recent {\sc PATE}
approach of \citep{PapernotAEGT17} can be adapted to give user-level
privacy, it is not suited for a language model where the number of
classes (possible output words) is large.

For problems like language modeling, protecting individual
examples is insufficient---each typed word makes its own contribution
to the RNN's training objective, so one user may contribute many
thousands of examples to the training data. A sensitive word or phrase
may be typed several times by an individual user, but it should still
be protected.\footnote{Differential privacy satisfies a property known
  as {\em group privacy} that can allow translation from example-level
  privacy to user-level privacy at the cost of an increased $\eps$. In
  our setting, such a blackbox approach would incur a prohibitive
  privacy cost. This forces us to directly address user-level
  privacy.}  In this work, we therefore apply the definition of
differential privacy to protect whole user histories in the training
set. This {\em user-level privacy} is ensured by using an appropriate
adjacency relation:
\begin{definition}User-adjacent datasets: Let $\database$ and
  $\database'$ be two datasets of training examples, where each
  example is associated with a user. Then, $\database$ and
  $\database'$ are \textbf{adjacent} if $\database'$ can be formed by
  adding or removing all of the examples associated with a single user
  from $\database$.
\end{definition}

Model training that satisfies differential privacy with respect to datasets
that are user-adjacent satisfies the intuitive notion
of privacy we aim to protect for language modeling: the presence or
absence of any specific user's data in the training set has an
imperceptible impact on the (distribution over) the parameters of the
learned model. It follows that an adversary looking at the trained
model cannot infer whether any specific user's data was used in the
training, irrespective of what auxiliary information they may have. In
particular, differential privacy rules out memorization of sensitive
information in a strong information theoretic sense.

%% file: algorithm.tex
\section{Algorithms for user-level differentially private training}
\label{sec:mechanism}

Our private algorithm relies heavily on two prior works: the
\fedavglong (or \fedavg) algorithm of \citet{mcmahan16fedavg}, which
trains deep networks on user-partitioned data, and the moments accountant
of \citet{abadi16dpdl}, which provides tight composition guarantees
for the repeated application of the Gaussian mechanism combined with
amplification-via-sampling. While we have attempted to make the
current work as self-contained as possible, the above references
provide useful background.

\fedavg was introduced by ~\citet{mcmahan16fedavg} for federated
learning, where the goal is to train a shared model while leaving the
training data on each user's mobile device. Instead, devices download
the current model and compute an update by performing local
computation on their dataset. It is worthwhile to perform extra
computation on each user's data to minimize the number of
communication rounds required to train a model, due to the
significantly limited bandwidth when training data remains
decentralized on mobile devices. We observe, however, that \fedavg is
of interest even in the datacenter when DP is applied: larger updates
are more resistant to noise, and fewer rounds of training can imply
less privacy cost. Most importantly, the algorithm naturally forms
per-user updates based on a single user's data, and these updates are
then averaged to compute the final update applied to the shared model
on each round. As we will see, this structure makes it possible to
extend the algorithm to provide a user-level differential privacy
guarantee.

We also evaluate the \fedsgdlong algorithm, essentially large-batch SGD
where each minibatch is composed of ``microbatches'' that include data
from a single distinct user.  In some datacenter applications \fedsgd
might be preferable to \fedavg, since fast networks make it more
practical to run more iterations. However, those additional iterations
come at a privacy cost. Further, the privacy benefits of federated
learning are nicely complementary to those of differential privacy,
and \fedavg can be applied in the datacenter as well, so we focus on
this algorithm while showing that our results also extend to \fedsgd. 

Both \fedavg and \fedsgd are iterative procedures, and in both cases we make the
following modifications to the non-private versions in order to achieve differential privacy:
\begin{enumerate}[A)]
\item We use random-sized batches where we select users independently
  with probability $q$, rather than always selecting a fixed number of
  users.
\item We enforce clipping of per-user updates so the total update has
  bounded $\ltwo$ norm.
\item  We use different estimators for the average update (introduced next).
\item We add Gaussian noise to the final average update.
\end{enumerate}
The pseudocode for \dpfedavg and \dpfedsgd is given as
Algorithm~\ref{alg:fedavg}.  \noiclr{Each of the changes above has the potential
to influence the convergence behavior of the algorithm, and so we
investigate them individually and together in
\secref{experiments}.}
In the remainder of this section, we introduce estimators for C) and
then different clipping strategies for B). Adding the sampling
procedure from A) and noise added in D) allows us to apply the moments
accountant to bound the total privacy loss of the algorithm, given in
Theorem~\ref{thm:privacy_guarantee}.  Finally, we consider the properties of the
moments accountant that make training on large datasets particular
attractive.

\newcommand{\spc}{$\quad$}
\newcommand{\comment}[1]{$\ $ // \emph{#1}}
\begin{figure}[t]
\begin{small}
\centering
\rule{\textwidth}{0.4pt}
\begin{minipage}[t]{0.49\textwidth} 
\begin{center}
\begin{algorithmic}
\SUB{Main training loop:}
   \STATE \emph{parameters}
   \STATE \spc user selection probability $q \in (0, 1]$
   \STATE \spc per-user example cap $\maxuserweight \in \R^+$
   \STATE \spc noise scale $z \in \R^+$
   \STATE \spc estimator $\fixedest$, or $\varyest$ with param $\Wmin$
   \STATE \spc UserUpdate (for FedAvg or FedSGD)
   \STATE \spc ClipFn (FlatClip or PerLayerClip)
   \STATE
   \STATE Initialize model $\mp^0$, moments accountant $\moments$
   \STATE $w_k = \min\!\left(\frac{n_k}{\maxuserweight}, 1\right)$ for all users $k$
   \STATE $W = \sum_{k \in d} w_k$ 
   \FOR{each round $t = 0, 1, 2, \dots$}
     \STATE $\sample^t \leftarrow$ (sample users with probability $q$)
     \FOR{each user $k \in \sample^t$ \textbf{in parallel}}
       \STATE $\update^{t+1}_k \leftarrow \text{UserUpdate}(k, \mp^t, \clipfn)$ 
     \ENDFOR
     \STATE $\update^{t+1} = 
        \begin{cases} 
          \frac{\sum_{k\in \sample^t} w_k \update_k}{qW} 
              & \text{for } \fixedest \\[4pt]
          \frac{\sum_{k \in \sample^t} w_k \update_k}{\max(q\Wmin,\sum_{k\in \sample^t} w_k)}
              & \text{for } \varyest
        \end{cases}$
     \STATE $S \leftarrow $ (bound on $\norm{\update_k}$ for \clipfn)
     \STATE $\sigma \leftarrow \left\{\text{$\frac{z S}{qW}$ for $\fixedest$ or $\frac{2zS}{q\Wmin}$ for $\varyest$} \right\}$
     \STATE $\mp^{t+1} \leftarrow \mp^t + \update^{t+1} + \Normal(0, I\sigma^2)$
     \STATE $\moments$.\texttt{accum\_priv\_spending}($z$)
   \ENDFOR
   \STATE print $\moments$.\texttt{get\_privacy\_spent}$()$
\end{algorithmic}
\end{center}
\vfill
\end{minipage}
\hfill
\begin{minipage}[t]{0.49\textwidth}
\begin{center}
\begin{algorithmic}
 \SUB{FlatClip$(\update)$:}
   \STATE \emph{parameter} $S$
   \STATE return $\clip{\update}{S}$ \comment{See Eq.~\eqref{eq:clipdef}.}

 \STATE
 \SUB{PerLayerClip$(\update)$}:
   \STATE \emph{parameters} $S_1, \dots S_\numlayers$ 
   \STATE $S = \sqrt{\sum_j S_j^2}$
   \FOR{each layer $j \in \{1, \dots, \numlayers\}$}
     \STATE $\update'(j) = \clip{\update(j)}{S_j}$
   \ENDFOR
   \STATE return $\update'$

 \STATE
 \SUB{UserUpdateFedAvg($k, \mp^0$ {\normalfont, ClipFn}):}
  \STATE \emph{parameters} $\lbs$, $\lepochs$, $\eta$
  \STATE $\mp \leftarrow \mp^0$
  \FOR{each local epoch $i$ from $1$ to $\lepochs$}
    \STATE $\mathcal{B} \leftarrow$ ($k$'s data split into size $\lbs$ batches)
    \FOR{batch $b \in \mathcal{B}$}
      \STATE $\mp \leftarrow \mp - \eta \grad \loss(\mp; b)$
      \STATE $\mp \leftarrow \mp^0 + \clipfn(\mp - \mp^0)$
    \ENDFOR
 \ENDFOR
 \STATE return update $\update_k = \mp - \mp^0$ \comment{Already clipped.}

 \STATE
 \SUB{UserUpdateFedSGD($k, \mp^0$ {\normalfont, ClipFn}):}
  \STATE \emph{parameters} $\lbs$, $\eta$
  \STATE select a batch $b$ of size $\lbs$ from $k$'s examples
  \STATE return update $\update_k = \clipfn(-\eta \grad \loss(\mp; b))$ 
\end{algorithmic}
\end{center}
\vfill
\end{minipage}
\end{small}

\rule{\textwidth}{0.4pt} 
\mycaptionof{algorithm}{The main loop for \dpfedavg and \dpfedsgd,
  the only difference being in the user update function
  (UserUpdateFedAvg or UserUpdateFedSGD).  The calls on the
  moments accountant $\moments$ refer to the API of
  \citet{abadi16dpdl-code}. \noiclr{The user update functions are
  parameterized by a clipping strategy and corresponding parameters,
  either FlatClip or PerLayerClip. Both use a local learning rate
  $\eta$ and a batch size $\lbs$, though typically $\lbs$ will be set
  much larger for FedSGD (for example, computing the gradient on the
  complete user dataset).}}\label{alg:fedavg}

\end{figure}

\paragraph{Bounded-sensitivity estimators for weighted average
  queries}
\label{sec:estimators}
Randomly sampling users (or training examples) by selecting each
independently with probability $q$ is crucial for proving low privacy
loss through the use of the moments
accountant~\citep{abadi16dpdl}. However, this procedure produces
variable-sized samples $\sample$, and when the quantity to be
estimated $f(\sample)$ is an average rather than a sum (as in computing the
weighted average update in \fedavg or the average loss on a minibatch
in SGD with example-level DP), this has ramifications for the
sensitivity of the query $f$.

Specifically, we consider weighted databases $d$ where each row $k \in
d$ is associated with a particular user, and has an associated weight
$w_k \in [0,1]$. This weight captures the desired influence of the row
on the final outcome. For example, we might think of row $k$
containing $n_k$ different training examples all generated by user
$k$, with weight $w_k$ proportional to $n_k$. We are then interested
in a bounded-sensitivity estimate of $f(\sample) = \frac{\sum_{k \in
    \sample} w_k \update_k}{\sum_{k\in \sample} w_k}$ for per-user
vectors $\update_k$, for example to estimate the weighted-average user
update in \fedavg. Let $W = \sum_k w_k$. We consider two such
estimators:
\begin{equation*}\label{eq:estimator}
  \fixedest(\sample) =\frac{\sum_{k\in \sample} w_k
    \update_k}{qW},
  \qquad \text{and} \qquad
\varyest(\sample) = \frac{\sum_{k \in \sample} w_k \update_k}{\max(q\Wmin,\sum_{k\in \sample} w_k)}.
\end{equation*}
Note $\fixedest$ is an unbiased estimator, since $\E[\sum_{k \in
  \sample} w_k] = qW$. On the other hand, $\varyest$ matches $f$
exactly as long as we have sufficient weight in the sample.  For
privacy protection, we need to control the sensitivity of our query
function $\tilde{f}$, defined as $\sensitivity(\tilde{f}) = \max_{\sample, k} \|
\tilde{f}(\sample\cup \{k\}) - \tilde{f}(\sample) \|_2$, where the added
user $k$ can have arbitrary data. The lower-bound $q\Wmin$ on the
denominator of $\varyest$ is necessary to control
sensitivity. Assuming each $w_k\Delta_k$ has bounded norm, we have:
\begin{lemma}\label{lem:sensitivity}
  If for all users $k$ we have $\|w_k \update_k\|_2\leq S$, then the
  sensitivity of the two estimators is bounded as
  $\sensitivity(\fixedest) \le \frac{S}{qW}$ and
  $\sensitivity(\varyest) \le \frac{2S}{q\Wmin}$.
\end{lemma}
A proof is given in Appendix~\secref{proofs}.

\paragraph{Clipping strategies for multi-layer models}
\label{sec:clipping}
Unfortunately, when the user vectors $\update_k$ are gradients (or sums
of gradients) from a neural network, we will generally have no a
priori bound\footnote{To control sensitivity,
  Lemma~\ref{lem:sensitivity} only requires that $\norm{w_k
    \update_k}$ is bounded. For simplicity, we only apply clipping to
  the updates $\update_k$, using the fact $w_k \le 1$, leaving as
  future work the investigation of weight-aware clipping schemes.}
$S$ such that $\norm{\update_k} \le S$. Thus, we will need to ``clip''
our updates to enforce such a bound before applying $\fixedest$ or
$\varyest$. For a single vector $\Delta$, we can apply a simple
$\ltwo$ projection when necessary:
\begin{equation}\label{eq:clipdef}
\clip{\update}{S} \eqdef \update \cdot \min\left(1, \frac{S}{\norm{\update}}\right).
\end{equation}
However, for deep networks it is more natural to treat the parameters
of each layer as a separate vector. The updates to each of these
layers could have vastly different $\ltwo$ norms, and so it can be
preferable to clip each layer separately.

Formally, suppose each update $\update_k$ contains $\numlayers$
vectors $\update_k = (\update_k(1), \dots, \update_k(\numlayers))$. We
consider the following clipping strategies, both of which ensure the
total update has norm at most $S$:
\begin{enumerate}
\item \textbf{Flat clipping} Given an overall clipping parameter $S$,
  we clip the concatenation of all the layers as $\update'_k =
  \clip{\update_k}{S}$.
\item \textbf{Per-layer clipping} Given a per-layer clipping parameter
  $S_j$ for each layer, we set $\update_k'(j) =
  \clip{\update_k(j)}{S_j}$. Let $S = \sqrt{\sum_{j=1}^m S_j^2}$. The
  simplest model-independent choice is to take $S_j =
  \frac{S}{\sqrt{m}}$ for all $j$, which we use in
  experiments.
\end{enumerate}

We remark here that clipping itself leads to additional bias, and
ideally, we would choose the clipping parameter to be large enough
that nearly all updates are smaller than the clip value. On the other
hand, a larger $S$ will require more noise in order to achieve
privacy, potentially slowing training. We treat $S$ as a
hyper-parameter and tune it.  
\noiclr{In our implementation of \dpfedavg, we apply ``greedy''
  clipping after each local update, following a standard approach for
  online gradient descent (see e.g.~\citet{zinkevich03giga}).}

\paragraph{A privacy guarantee}
\label{sec:guarantee}

Once the sensitivity of the chosen estimator is bounded, we may add
Gaussian noise scaled to this sensitivity to obtain a privacy
guarantee. A simple approach is to use an $(\eps,\delta)$-DP bound for
this Gaussian mechanism, and apply the privacy amplification lemma and
the advanced composition theorem to get a bound on the total privacy
cost. We instead use the Moments Accountant of~\citet{abadi16dpdl} to
achieve much tighter privacy bounds. The moments accountant for the
sampled Gaussian mechanism upper bounds the total privacy cost of $T$
steps of the Gaussian mechanism with noise $N(0,\sigma^2)$ for $\sigma
= \noisescale \cdot \sensitivity$, where $\noisescale$ is a parameter,
$\sensitivity$ is the sensitivity of the query, and each row is
selected with probability $q$. Given a $\delta > 0$, the accountant
gives an $\eps$ for which this mechanism satisfies $(\eps,\delta)$-DP.
The following theorem is a slight generalization of the results
in~\cite{abadi16dpdl}; see \secref{proofs} for a proof sketch.
\begin{theorem}\label{thm:privacy_guarantee}
  For the estimator ($\fixedest$,$\varyest$), the moments accountant
  of the sampled Gaussian mechanism correctly computes the privacy
  loss with the noise scale of $\noisescale = \sigma/\sensitivity$ and
  steps $T$, where $\sensitivity=S/qW$ for ($\fixedest$) and
  $2S/q\Wmin$ for ($\varyest$).
\end{theorem}

\paragraph{Differential privacy for large datasets}
\label{sec:large}
We use the implementation of the moments accountant from
\citet{abadi16dpdl-code}. The moments accountant makes strong use of
amplification via sampling, which means increasing dataset size makes
achieving high levels of privacy significantly
easier. Table~\ref{table:moments} summarizes the privacy guarantees
offered as we vary some of the key parameters.
\begin{table}[t!]
\begin{center}
  \caption{ Privacy for different total numbers of users $\numusers$
    (all with equal weight), expected number of users sampled per
    round $\eC$, and the number of rounds of training. For each row,
    we set $\delta=\frac{1}{\numusers^{1.1}}$ and report the value of
    $\eps$ for which $(\eps, \delta)$-differential privacy holds after
    $1$ to $10^6$ rounds. For large datasets, additional
    rounds of training incur only a minimal additional privacy loss.
  }
\label{table:moments}
\vspace{4pt}
\begin{small}
\begin{tabular}{|rrr|rrrrrrr|}
\hline
users         & sample                & noise       & 
\multicolumn{7}{c|}{\T Upper bound on privacy $\eps$ after $1, 10, \dots 10^6$ rounds} \\
$K$   & $\eC$               & $\noisescale$ & $10^0$ & $10^1$ & $10^2$ & $10^3$ & $10^4$ & $10^5$ & $10^6$ \\
\hline
\T$10^5$ & $10^2$ & 1.0 &    0.97 &  0.98 &  1.00 &  1.07 &  1.18 &  2.21 &  7.50 \\
$10^6$ & $10^1$ & 1.0 &    0.68 &  0.69 &  0.69 &  0.69 &  0.69 &  0.72 &  0.73 \\
$10^6$ & $10^3$ & 1.0 &    1.17 &  1.17 &  1.20 &  1.28 &  1.39 &  2.44 &  8.13 \\
$10^6$ & $10^4$ & 1.0 &    1.73 &  1.92 &  2.08 &  3.06 &  8.49 & 32.38 & 187.01 \\
$10^6$ & $10^3$ & 3.0 &    0.47 &  0.47 &  0.48 &  0.48 &  0.49 &  0.67 &  1.95 \\
$10^9$ & $10^3$ & 1.0 &    0.84 &  0.84 &  0.84 &  0.85 &  0.88 &  0.88 &  0.88 \\
\hline
\end{tabular}
\end{small}
\vspace{0.1in}
\end{center}
\end{table}
The takeaway from this table is that as long as we can afford the cost
in utility of adding noise proportional to $\noisescale$ times the
sensitivity of the updates, we can get reasonable privacy guarantees
over a large range of parameters. The size of the dataset has a modest
impact on the privacy cost of a single query (1 round column), but a
large effect on the number of queries that can be run without
significantly increasing the privacy cost (compare the $10^6$ round
column). For example, on a dataset with $10^9$ users, the privacy
upper bound is nearly constant between $1$ and $10^6$ calls to the
mechanism (that is, rounds of the optimization algorithm).

There is only a small cost in privacy for increasing the expected
number of (equally weighted) users $\eC = qW = q\numusers$ selected on each round
as long as $\eC$ remains a small fraction of the size of the total
dataset. Since the sensitivity of an average query decreases like
$1/\eC$ (and hence the amount of noise we need to add
decreases proportionally), we can increase $\eC$ until we
arrive at a noise level that does not adversely effect the
optimization process. We show empirically that such a level exists in
the experiments.

%% file: experiments.tex
\section{Experimental Results}
\label{sec:experiments}

In this section, we evaluate \dpfedavg while training an LSTM RNN
tuned for language modeling in a mobile keyboard. We vary noise,
clipping, and the number of users per round to develop an intuition of
how privacy affects model quality in practice. 

We defer our experimental results on \fedsgd as well as on models with
larger dictionaries to Appendix~\secref{additional}. To summarize,
they show that \fedavg gives better privacy-utility trade-offs than
\fedsgd, and that our empirical conclusions extend to larger
dictionaries with relatively little need for additional parameter
tuning despite the significantly larger models. Some less important
plots are deferred to \secref{moreplots}.

\newcommand{\model}[2]{\textbf{\textcolor{#1}{#2}}}
\setlength{\pw}{2.8in}
\begin{figure}
\begin{minipage}[t]{.48\textwidth}
  \begin{center}
    \includegraphics[width=\pw]{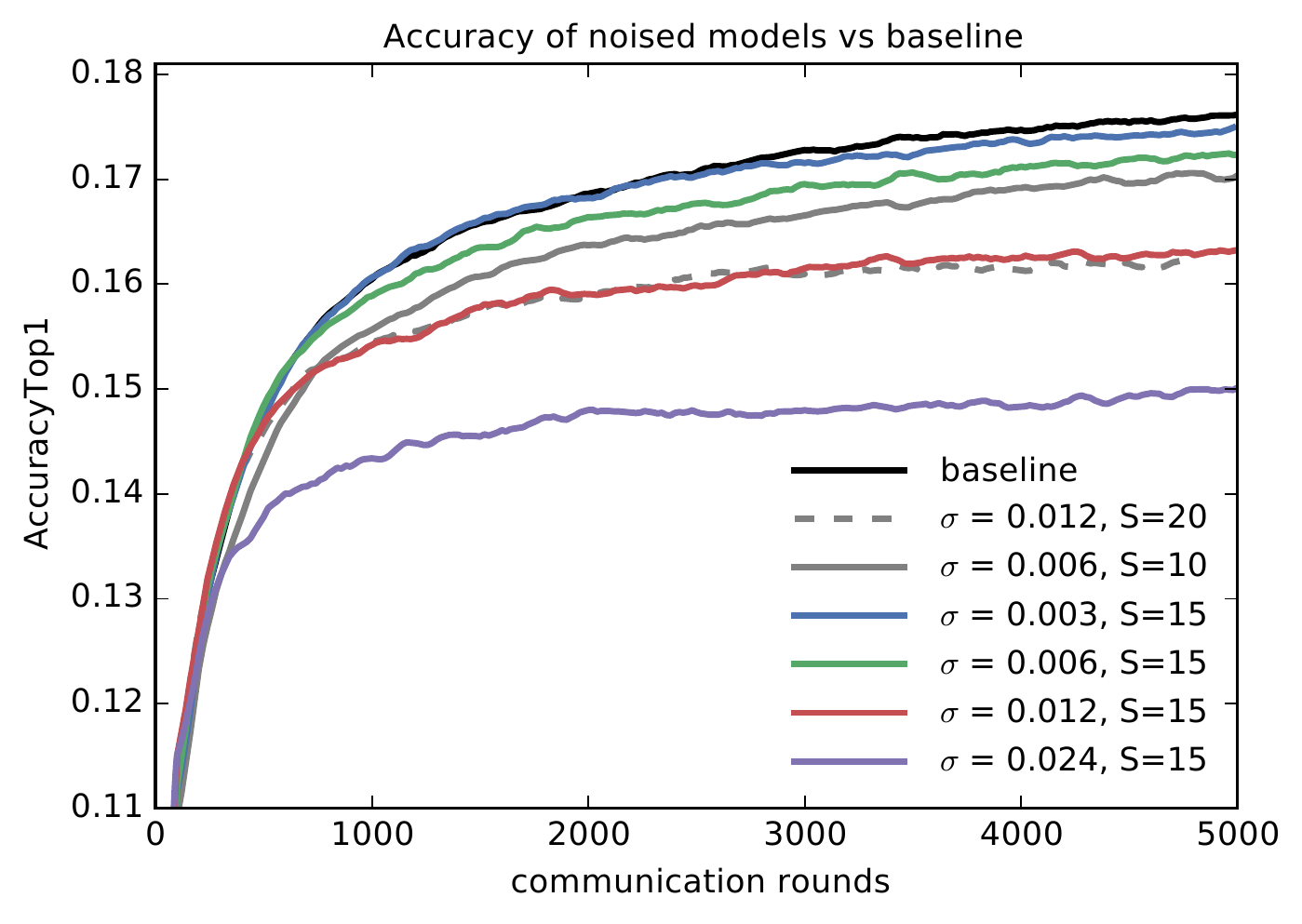} \\
    \mycaptionof{figure}{Noised training versus the non-private baseline.
      \noiclr{ of Figure~\ref{fig:estimators_and_sampling}.}
      The model with $\sigma=0.003$ nearly matches the baseline.}
    \label{fig:noised_vs_baseline}
  \end{center}
\end{minipage}
\hfill
\begin{minipage}[t]{.49\textwidth}
\begin{center}
  \mycaptionof{table}{Privacy ($\eps$ at $\delta\!=\!10^{-9}$) and
    accuracy after 5000 rounds of training for models with different
    $\sigma$ and $S$ from Figure~\ref{fig:noised_vs_baseline}. The
    $\eps$'s are strict upper bounds on the true privacy loss given
    the dataset size $\numusers$ and $\eC$; AccuracyTop1 (AccT1) is
    estimated from a model trained with the same $\sigma$ as discussed
    in the text.}\label{table:dpaccuracy}
\begin{small}
\begin{tabular}{|rr|rr|rr|}
\hline
\multicolumn{2}{|c|}{model} & \multicolumn{2}{c|}{data}    &  & \\
$\sigma$ & $S$              & users $\numusers$ & $\eC$  & $\eps$  &  AccT1\\
\hline
\model{black}{0.000} & \model{black}{$\infty$} & 763430 & 100 & $\infty$  & 17.62\% \\
\model{modelblue}{0.003} & \model{modelblue}{15} & 763430 & 5000 & 4.634 & 17.49\% \\
\model{modelgray}{0.006} & \model{modelgray}{10} & 763430 & 1667 & 2.314 & 17.04\% \\
\model{modelred}{0.012} & \model{modelred}{15} & 763430 & 1250 & 2.038 & 16.33\% \\
\model{modelblue}{0.003} & \model{modelblue}{15} &$10^8$ & 5000 & 1.152 & 17.49\% \\
\model{modelgray}{0.006} & \model{modelgray}{10} &$10^8$ & 1667 & 0.991 & 17.04\% \\
\model{modelred}{0.012} & \model{modelred}{15} &$10^8$ & 1250 & 0.987 & 16.33\% \\
\hline
\end{tabular}
\end{small}
\end{center}
\end{minipage}
\end{figure}

\paragraph{Model structure}

\newcommand{\word}{s}
The goal of a language model is to predict the next word in a sequence
$\word_t$ from the preceding words $\word_0 ... \word_{t-1}$. The
neural language model architecture used here is a variant of the LSTM
recurrent neural network \citep{hochreiter1997long} trained to predict
the next word (from a fixed dictionary) given the current word and a
state vector passed from the previous time step. LSTM language models
are competitive with traditional n-gram models
\citep{sundermeyer2012lstm} and are a standard baseline for a variety
of ever more advanced neural language model architectures
\citep{grave2016improving, merity2016pointer,
  gal2016theoretically}. Our model uses a few tricks to decrease the
size for deployment on mobile devices (total size is 1.35M
parameters), but is otherwise standard. We evaluate using
\texttt{AccuracyTop1}, the probability that the word to which the
model assigns highest probability is correct \noiclr{(after some
  minimal normalization)}. Details on the model and evaluation metrics
  are given in \secref{details}. All training began from a common
  random initialization, though for real-world applications
  pre-training on public data is likely preferable (see
  \secref{details} for additional discussion).

\paragraph{Dataset}
We use a large public dataset of Reddit posts, as described
by~\citet{alrfou16reddit}. Critically for our purposes, each post in
the database is keyed by an author, so we can group the data by these
keys in order to provide user-level privacy. We preprocessed the
dataset to $\numusers=763,430$ users each with 1600 tokens. Thus, we
take $w_k = 1$ for all users, so $W = \numusers$. We write $\eC =
q\numusers = qW$ for the expected number of users sampled per
round. See \secref{details} for details on the dataset and
preprocessing. To allow for frequent evaluation, we use a relatively
small test set of 75122 tokens formed from random held-out posts. We
evaluate accuracy every 20 rounds and plot metrics smoothed over 5
evaluations (100 rounds).

\setlength{\pw}{2.8in}
\begin{figure}
\begin{minipage}[t]{.48\textwidth}
  \begin{center}
    \includegraphics[width=\pw]{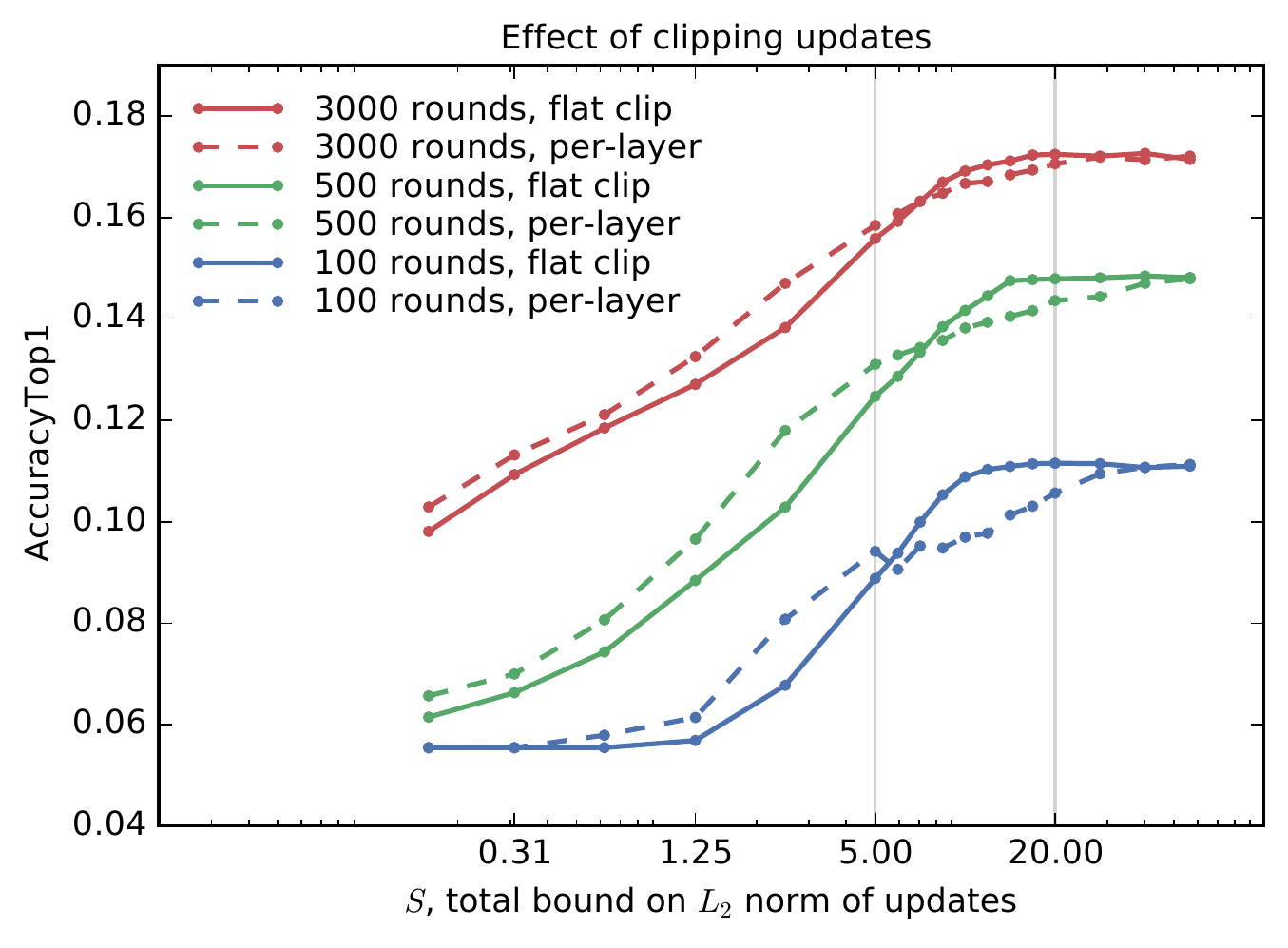} \\
    \mycaptionof{figure}{The effect of update clipping on the
      convergence of \fedavg, after 100, 500, and 3000 rounds of
      training.}
    \label{fig:effect_of_clipping}
  \end{center}
\end{minipage}
\hfill
\begin{minipage}[t]{.49\textwidth}
  \begin{center}
    \includegraphics[width=\pw]{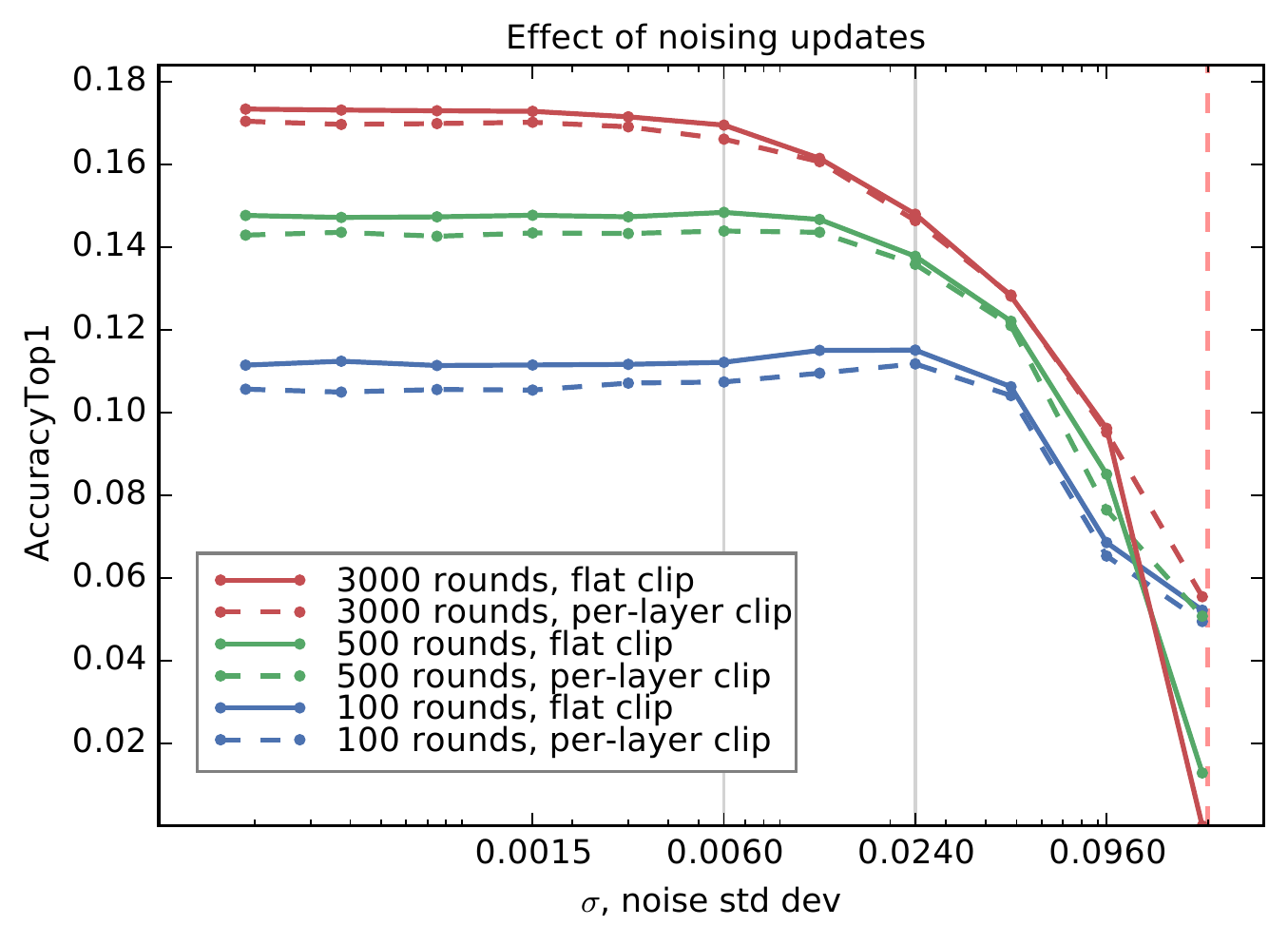} \\
    \mycaptionof{figure}{The effect of different levels of noise
      $\sigma$ for flat and per-layer clipping at $S=20$. The vertical
      dashed red line is $\sigma=0.2$.}
    \label{fig:effect_of_noise}
  \end{center}
\end{minipage}
\end{figure}

\begin{figure}
\begin{minipage}[t]{.49\textwidth}
  \begin{center}
  \end{center}
\end{minipage}
\vspace{-0.2in}
\end{figure}

\paragraph{Building towards DP: sampling, estimators, clipping, and
  noise}
Recall achieving differential privacy for \fedavg required
a number of changes (\secref{mechanism}, items A--D).  \noiclr{ 1) Use
  random-sized batches where we select users independently with
  probability $q$, rather than always selecting a fixed number; 2)
  Choose an estimator for the average update, which might, for
  example, increase variance by using the expected weight of the
  sample rather than the actual weight; 3) Enforce clipping of user
  updates so the total update has bounded $\ltwo$ norm; and, 4) Add
  sufficient Gaussian noise to the average update.}
In this section, we examine the impact of each of these changes, both
to understand the immediate effects and to enable the selection of
reasonable parameters for our final DP experiments. This sequence of
experiments also provides a general road-map for applying
differentially private training to new models and datasets. For these
experiments, we use the \fedavg algorithm with a fixed learning rate
of 6.0, which we verified was a reasonable choice in preliminary
experiments.\footnote{The proper choice of for the clipping parameters
  may depend on the learning rate, so if the learning rate is changed,
  clipping parameter choices will also need to be re-evaluated.} In
all \fedavg experiments, we used a local batch size of $\lbs=8$, an
unroll size of 10 tokens, and made $\lepochs=1$ passes over the local
dataset; thus \fedavg processes $80$ tokens per batch, processing a
user's 1600 tokens in 20 batches per round.

First, we investigate the impact of changing the estimator used for
the average per-round update, as well as replacing a fixed sample of
$\usersperround=100$ users per round to a variable-sized sample formed
by selecting each user with probability $q=100/763430$ for an
expectation of $\eC=100$ users. None of these changes significantly
impacted the convergence rate of the algorithm (see
Figure~\ref{fig:estimators_and_sampling} in \secref{moreplots}). In
particular, the fixed denominator estimator $\fixedest$ works just as
well as the higher-sensitivity clipped-denominator estimator
$\varyest$. Thus, in the remaining experiments we focus on estimator
$\fixedest$.

Next, we investigate the impact of flat and per-layer clipping on the
convergence rate of \fedavg. The model has 11 parameter vectors, and
for per-layer clipping we simply chose to distribute the clipping
budget equally across layers with $S_j =
S/\sqrt{11}$. Figure~\ref{fig:effect_of_clipping} shows that
choosing $S \in [10, 20]$ has at most a small effect on convergence
rate. \noiclr{In this regime, flat clipping slightly outperforms per-layer
clipping.}

Finally, Figure~\ref{fig:effect_of_noise} shows the impact of various
levels of per-coordinate Gaussian noise $\Normal(0, \sigma^2)$ added
to the average update. Early in training, we see almost no loss in
convergence for a noise of $\sigma=0.024$; later in training noise has
a larger effect, and we see a small decrease in convergence past
$\sigma = 0.012$.
These experiments, where we sample only an expected 100 users per
round, are not sufficient to provide a meaningful privacy
guarantee. We have $S = 20.0$ and $\eC=qW=100$, so the sensitivity of
estimator $\fixedest$ is $20/100.0 = 0.2$. Thus, to use the moments
accountant with $\noisescale=1$, we would need to add noise
$\sigma=0.2$ (dashed red vertical line), which destroys accuracy.

\paragraph{Estimating the accuracy of private models for large datasets}
Continuing the above example, if instead we choose $q$ so $\eC=1250$,
set the $\ltwo$ norm bound $S=15.0$, then we have sensitivity $15/1250 =
0.012$, and so we add noise $\sigma = 0.012$ and can apply the moments
account with noise scale $\noisescale=1$. The computation is now
significantly more computationally expensive, but will give a
guarantee of $(1.97, 10^{-9})$-differential privacy after 3000
rounds of training.
Because running such experiments is so computationally expensive, for
experimental purposes it is useful to ask: does using an expected
$1250$ users per round produce a model with different
\emph{accuracy} than a model trained with only $100$ expected users
per round? If the answer is no, we can train a model with $\eC=100$
and a particular noise level $\sigma$, and use that model to estimate
the utility of a model trained with a much larger $q$ (and hence a
much better privacy guarantee). We can then run the moments accountant
(without actually training) to numerically upper bound the privacy
loss.
To test this, we trained two models, both with $S=15$ and
$\sigma=0.012$, one with $\eC=100$ and one with $\eC=1250$; recall the
first model achieves a vacuous privacy guarantee, while the second
achieves $(1.97, 10^{-9})$-differential privacy after 3000
rounds. Figure~\ref{fig:noised_100_vs_1250_training} in
\secref{moreplots} shows the two models produce almost identical
accuracy curves during training. Using this observation, we can use
the accuracy of models trained with $\eC=100$ to estimate the utility
of private models trained with much larger $\eC$. See also
Figure~\ref{fig:users_per_round} in \secref{moreplots}, which also
shows diminishing returns for larger $C$ for the standard \fedavg
algorithm.

Figure~\ref{fig:noised_vs_baseline} compares the true-average
fixed-sample baseline model (see
Figure~\ref{fig:estimators_and_sampling} in \secref{moreplots}) with
models that use varying levels of clipping $S$ and noise $\sigma$ at
$\eC=100$. Using the above approach, we can use these experiments to
estimate the utility of LSTMs trained with differential privacy for
different sized datasets and different values of
$\eC$. Table~\ref{table:dpaccuracy} shows representative values
setting $\eC$ so that $\noisescale=1$. For example, the model with
$\sigma=0.003$ and $S=15$ is only worse than the baseline by an
additive $-0.13\%$ in AccuracyTop1 and achieves $(4.6,
10^{-9})$-differential privacy when trained with $\eC=5000$ expected
users per round. As a point of comparison, we have observed that
training on a different corpus can cost an additive $-2.50\%$ in
AccuracyTop1.\footnote{This experiment was performed on different
  datasets, comparing training on a dataset of public social media
  posts to training on a proprietary dataset which is more
  representative of mobile keyboard usage, and evaluating on a
  held-out sample of the same representative dataset. Absolute
  AccuracyTop1 was similar to the values we report here for the Reddit
  dataset.}

\begin{figure}[t]
\begin{minipage}[t]{.49\textwidth}
  \begin{center}
    \includegraphics[width=\pw]{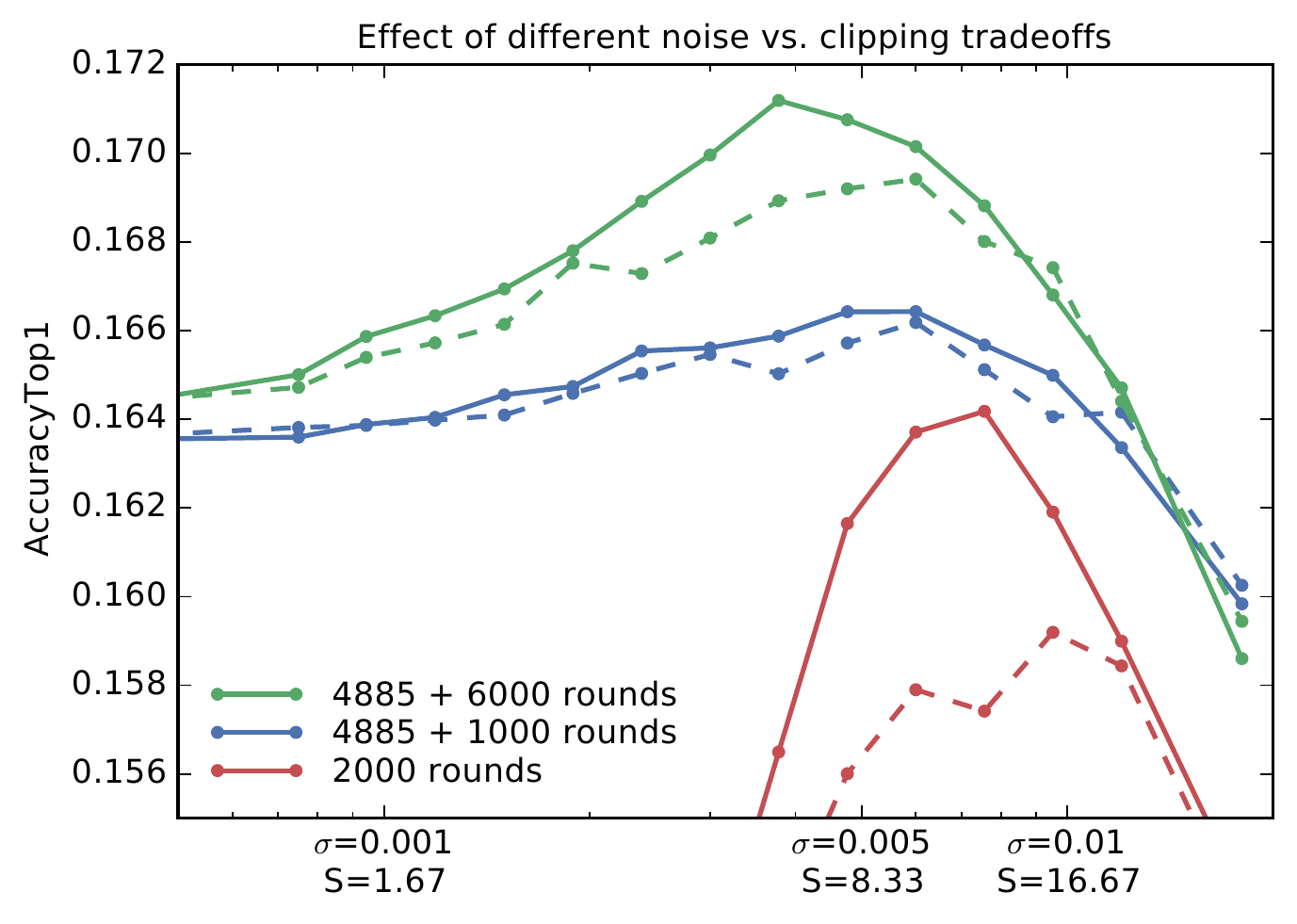} \\
    \mycaptionof{figure}{Different noise/clipping tradeoffs (all of
      equal privacy cost), for initial training (red) and adjusted
      after 4885 rounds (green and blue). Solid lines use flat
      clipping, dashed are per-layer.}
   \label{fig:continued_effect}
   \vfill
  \end{center}
  \vfill
\end{minipage}
\hfill
\begin{minipage}[t]{.49\textwidth}
  \begin{center}
    \includegraphics[width=\pw]{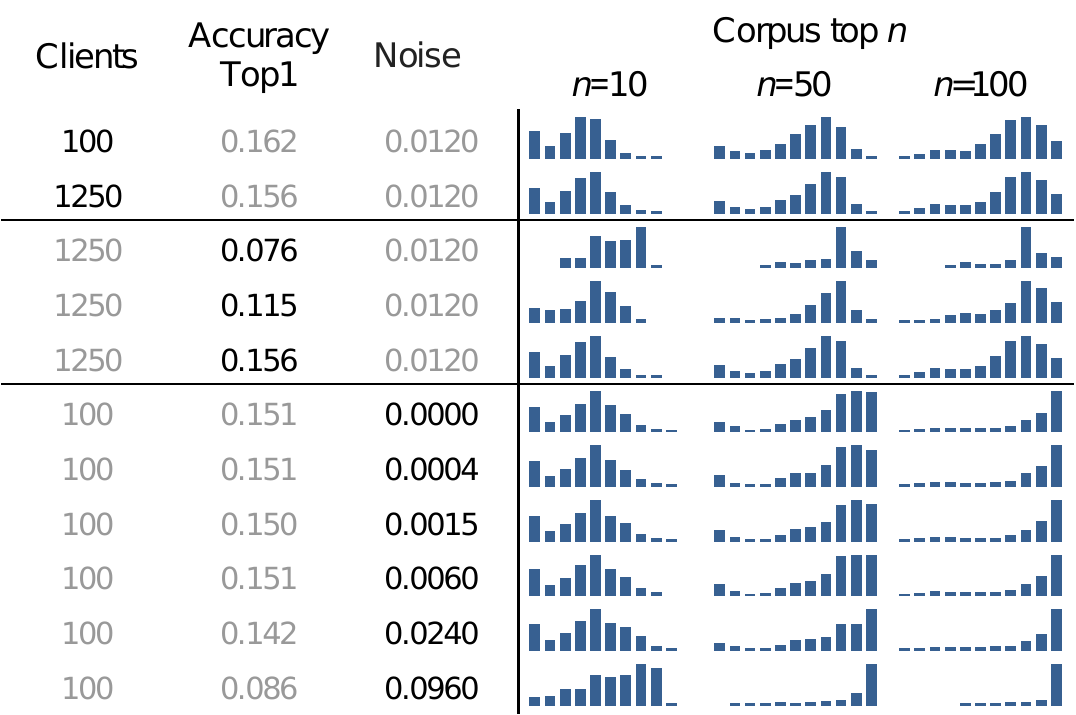} \\
    \mycaptionof{table}{Count histograms recording how many of a
      model's (row's) top 10 predictions are found in the $n=10$,
      $50$, or $100$ most frequent words in the corpus. Models that
      predict corpus top-$n$ more frequently have more mass to the
      right.}
    \label{table:lm_headedness}
  \end{center}
\end{minipage}
	\vspace{-0.2in}
\end{figure}

\paragraph{Adjusting noise and clipping as training progresses}
Figure~\ref{fig:noised_vs_baseline} shows that as training progresses,
each level of noise eventually becomes detrimental (the line drops
somewhat below the baseline).\noiclr{: $\sigma = 0.024$ has no effect
  until AccuracyTop1 reaches about $13\%$ but slows convergence after
  that, $\sigma=0.012$ doesn't hurt until AccuracyTop1 reaches
  $14.5\%$, and so on. }
This suggests using a smaller $\sigma$ and correspondingly smaller $S$
(thus fixing $\noisescale$ so the privacy cost of each round is
unchanged) as training progresses. Figure \ref{fig:continued_effect}
(and Figure \ref{fig:continued_training} in \secref{moreplots}) shows
this can be effective.  We indeed observe that early in training
(red), $S$ in the $10$ -- $12.6$ range works well ($\sigma = 0.006$ --
$0.0076$). However, if we adjust the clipping/noise tradeoff after
4885 rounds of training and continue for another 6000, switching to
$S=7.9$ and $\sigma=0.0048$ performs better.

\paragraph{Comparing DP and non-DP models}
While noised training with \dpfedavg has only a small effect on
predictive accuracy, it could still have a large qualitative effect on
predictions. We hypothesized that noising updates might bias the model
away from rarer words (whose embeddings get less frequent actual
updates and hence are potentially more influenced by noise) and toward
the common ``head'' words.
To evaluate this hypothesis, we computed predictions on a sample of
the test set using a variety of models. At each $\word_t$ we intersect
the top 10 predictions with the most frequent ${10, 50, 100}$ words in
the dictionary. So for example, an intersection of size two in the top
50 means two of the model's top 10 predictions are in the 50 most
common words in the dictionary.  Table~\ref{table:lm_headedness} gives
histograms of these counts.
We find that better models (higher AccuracyTop1) tend to use fewer
head words, but see little difference from changing $\eC$ or the noise
$\sigma$ (until, that is, enough noise has been added to compromise
model quality, at which point the degraded model's bias toward the
head matches models of similar quality with less noise).

%% file: algorithm-supp.tex
\section{Additional proofs}\label{sec:proofs}

\begin{proof}[Proof of Lemma~\ref{lem:sensitivity}]
  For the first bound, observe the numerator in the estimator
  $\fixedest$ can change by at most $S$ between neighboring databases,
  by assumption. The denominator is a constant.  For the second bound,
  the estimator $\varyest$ can be thought of as the sum of the vectors
  $w_k \update_k$ divided by $\max(qW_{min}, \sum_{k \in \sample}
	\update_k)$. Writing $\Num(\sample)$ for the numerator $\sum_{k \in \sample} w_k \update_k$, and $\Den(\sample)$ for the denominator $\max(q\Wmin, \sum_{k \in \sample} w_k)$, the following are immediate for any $\sample$ and $\sample' \eqdef \sample \cup \{k\}$:
	\begin{align*}
		\|\Num(\sample') - \Num(\sample)\| &= \| w_k\update_k\| \leq S.\\
		\|\Den(\sample') - \Den(\sample)\| & \leq 1.\\
		\|\Den(\sample')\| &\geq q\Wmin.
	\end{align*}
  It follows that
	\begin{align*}
		\|\varyest(\sample') - \varyest(\sample)\| &= \left\| \frac{\Num(\sample')}{\Den(\sample')} - \frac{\Num(\sample)}{\Den(\sample)}\right\|\\
		&= \left\| \frac{\Num(\sample') - \Num(\sample)}{\Den(\sample')} + \Num(\sample)\left(\frac{1}{\Den(\sample')} - \frac{1}{\Den(\sample)}\right)\right\|\\
		&\leq \left\|\frac{w_k\update_k}{\Den(\sample')}\right\| + \left\|\frac{\Num(\sample)}{\Den(\sample)}\left(\frac{\Den(\sample) - \Den(\sample')}{\Den(\sample')}\right)\right\|\\
		&\leq \frac{S}{q\Wmin} + \|\varyest(\sample)\|\left(\frac{1}{q\Wmin}\right)\\
		&\leq \frac{2S}{q\Wmin}.
	\end{align*}
	Here in the last step, we used the fact that $\|\varyest(\sample)\| \leq S$. The claim follows.
\end{proof}

\begin{proof}[Proof of Theorem~\ref{thm:privacy_guarantee}]
It suffices to verify that 1. the moments (of the privacy loss) at each
step are correctly bounded; and, 2. the composability holds when
accumulating the moments of multiple steps.

At each step, users are selected randomly with probability $q$. If in
addition the $\ltwo$-norm of each user's update is upper-bounded by
$\sensitivity$, then the moments can be upper-bounded by that of the
sampled Gaussian mechanism with sensitivity $1$, noise scale
$\sigma/\sensitivity$, and sampling probability $q$.

Our algorithm, as described in Figure~\ref{alg:fedavg}, uses a fixed
noise variance and generates the i.i.d. noise independent of the
private data. Hence we can apply the composability as in Theorem~2.1
in~\cite{abadi16dpdl}.

We obtain the theorem by combining the above and the sensitivity bounds
$\fixedest$ and $\varyest$.
\end{proof}

%% file: experiments_supplement.tex
\section{Experiment details}\label{sec:details}

\paragraph{Model}
The first step in training a word-level recurrent language model is
selecting the vocabulary of words to model, with remaining words
mapped to a special ``UNK'' (unknown) token.  Training a fully
differentially private language model from scratch requires a private
mechanism to discover which words are frequent across the corpus, for
example using techniques like distributed heavy-hitter estimation
\citep{chan2012differentially,bassily17practical}.  For this work, we
simplified the problem by pre-selecting a dictionary of the most
frequent 10,000 words (after normalization) in a large corpus of mixed
material from the web and message boards (but not our training or test
dataset).

Our recurrent language model works as follows: word $\word_t$ is
mapped to an embedding vector $e_t\in\R^{96}$ by looking up the word
in the model's vocabulary. The $e_t$ is composed with the state
emitted by the model in the previous time step $s_{t-1}\in\R^{256}$ to
emit a new state vector $s_t$ and an ``output embedding''
$o_t\in\R^{96}$. The details of how the LSTM composes $e_t$ and
$s_{t-1}$ can be found in \cite{hochreiter1997long}. The output
embedding is scored against the embedding of each item in the
vocabulary via inner product, before being normalized via softmax to
compute a probability distribution over the vocabulary.
Like other standard language modeling applications, we treat every
input sequence as beginning with an implicit ``BOS'' (beginning of
sequence) token and ending with an implicit ``EOS'' (end of sequence)
token.

Unlike standard LSTM language models, our model uses the same learned
embedding for the input tokens and for determining the predicted
distribution on output tokens from the
softmax.\footnote{\citet{ofir17output} independently introduced this
  technique and provide an empirical analysis comparing models with
  and without weight tying.} This reduces the size of the model by
about 40\% for a small decrease in model quality, an advantageous
tradeoff for mobile applications. Another change from many standard
LSTM RNN approaches is that we train these models to restrict the word
embeddings to have a fixed $\ltwo$ norm of 1.0, a modification found
in earlier experiments to improve convergence time.  In total the
model has 1.35M trainable parameters.

\paragraph{Initialization and personalization}

For many applications public proxy data is available, e.g., for
next-word prediction one could use public domain books, Wikipedia
articles, or other web content. In this case, an initial model trained
with standard (non-private) algorithms on the public data (which is
likely drawn from the wrong distribution) can then be further refined
by continuing with differentially-private training on the private data
for the precise problem at hand. Such pre-training is likely the best
approach for practical applications. However, since training models
purely on private data (starting from random initialization) is a
strictly harder problem, we focus on this scenario for our
experiments.

Our focus is also on training a single model which is shared by all
users. However, we note that our approach is fully compatible with
further on-device personalization of these models to the particular
data of each user. It is also possible to give the central model some
ability to personalize simply by providing information about the user
as a feature vector along with the raw text input. LSTMs are
well-suited to incorporating such additional context.

\paragraph{Accuracy metrics}
We evaluate using \texttt{AccuracyTop1}, the probability that the word
to which the model assigns highest probability is correct (after some
minimal normalization). We always count it as a mistake if the true
next word is not in the dictionary, even if the model predicts UNK, in
order to allow fair comparisons of models using different
dictionaries. In our experiments, we found that our model
architecture is competitive on \texttt{AccuracyTop1} and related
metrics (Top3, Top5, and perplexity) across a variety of tasks and
corpora.
\noiclr{Our experiments focus on AccuracyTop1 because that metric is
simple to interpret and is well correlated with improvements to the
typing experience on mobile keyboards.}

\paragraph{Dataset}
The Reddit dataset can be accessed through Google
BigQuery \citep{reddit_dataset}. Since our goal is to limit the
contribution of any one author to the final model, it is not necessary
to include all the data from users with a large number of posts. On
the other hand, processing users with too little data slows
experiments (due to constant per-user overhead). Thus, we use a
training set where we have removed all users with fewer than 1600
tokens (words), and truncated the remaining $\numusers=763,430$ users
to have exactly 1600 tokens.

We intentionally chose a public dataset for research purposes, but
carefully chose one with a structure and contents similar to private
datasets that arise in real-world language modeling task such as
predicting the next-word in a mobile keyboard. This allows for
reproducibility, comparisons to non-private models, and inspection of
the data to understand the impact of differential privacy beyond
coarse aggregate statistics (as in Table~\ref{table:lm_headedness}).

\clearpage
\section{Supplementary Plots}\label{sec:moreplots}

\setlength{\pw}{2.8in}
\begin{figure}[H]
   \begin{center}
     \includegraphics[width=\pw]{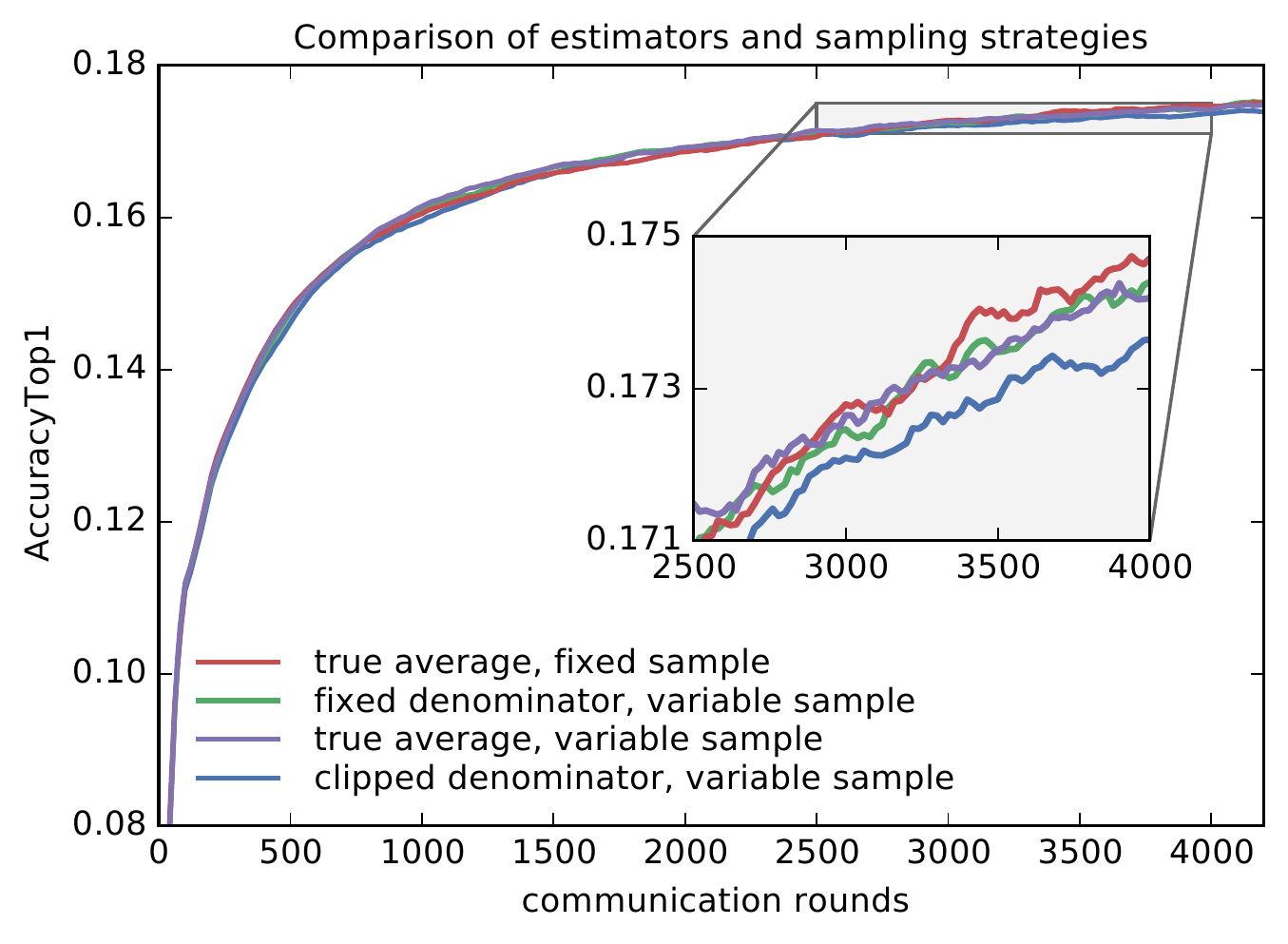} \\
     \vskip -0.06in \mycaptionof{figure}{Comparison of sampling
       strategies and estimators. \emph{Fixed sample} is exactly
       $C=100$ users per round, and \emph{variable sample} selects
       uniformly with probability $q$ for $\eC=100$. The \emph{true
         average} corresponds to $f$, \emph{fixed denominator} is
       $\fixedest$, and \emph{clipped denominator} is $\varyest$.}
    \label{fig:estimators_and_sampling}
  \end{center}
\end{figure}
\vspace{-0.2in}

\begin{figure}[H]
\begin{minipage}[t]{.49\textwidth}
  \begin{center}
    \includegraphics[width=\pw]{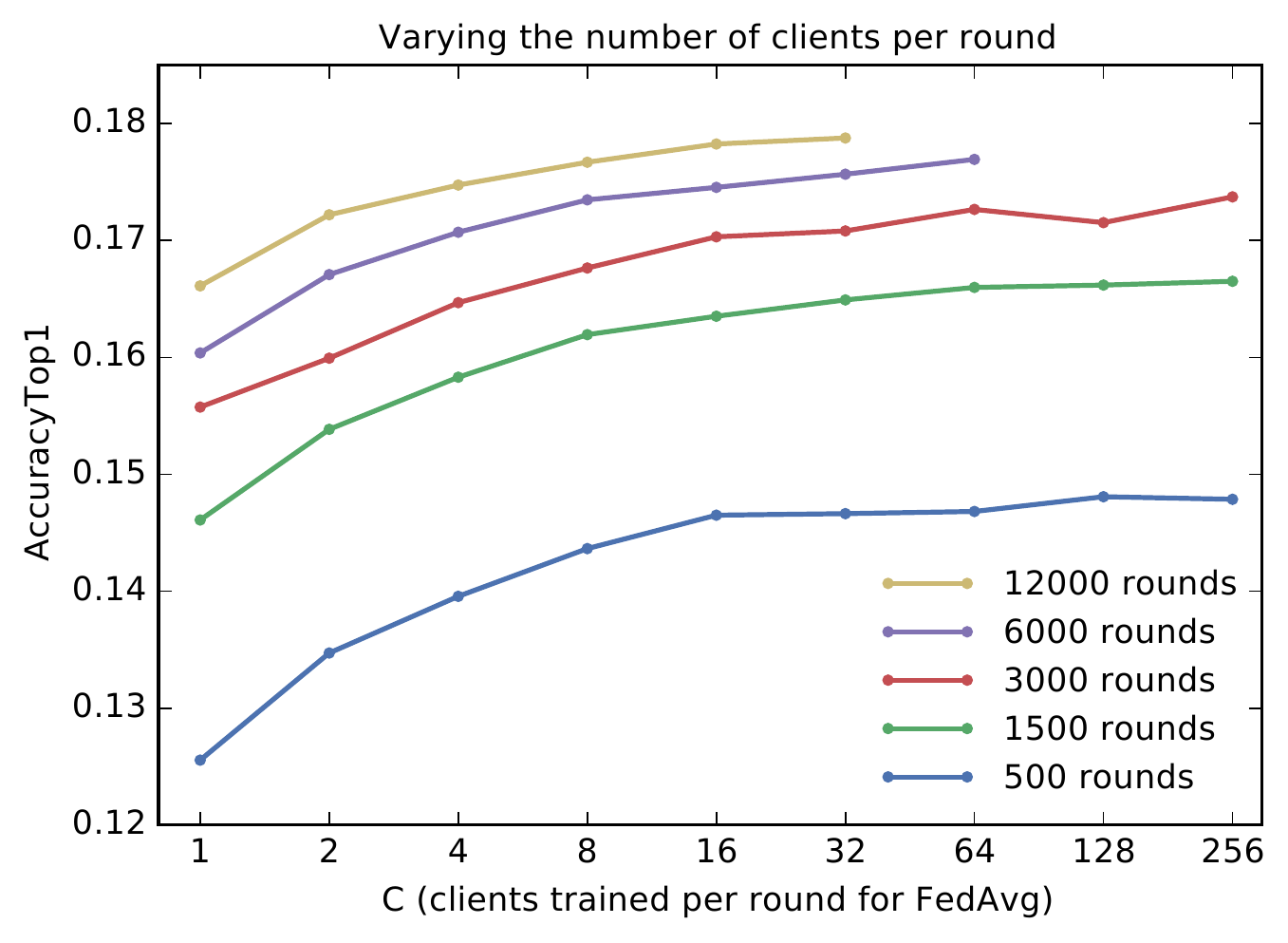} \\
    \mycaptionof{figure}{The effect of $C$ for \fedavg using the exact
      estimator and without noise or clipping.}
    \label{fig:users_per_round}
  \end{center}
\end{minipage}
\hfill
\begin{minipage}[t]{.49\textwidth}
  \begin{center}
    \includegraphics[width=\pw]{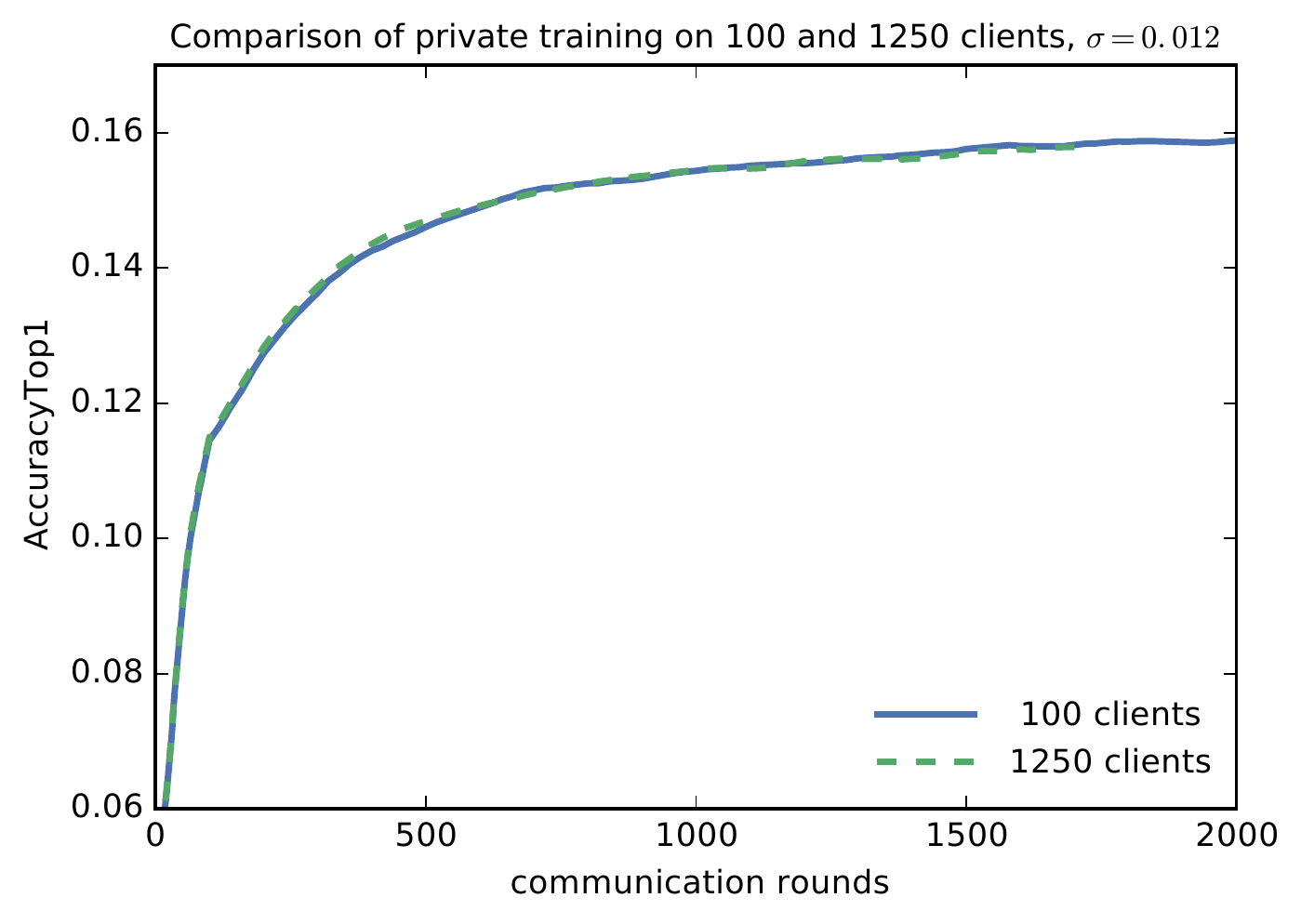} \\
    \mycaptionof{figure}{Training with (expected) 100 vs 1250 users
      per round, both with flat-clipping at $S=15.0$ and
      per-coordinate noise with $\sigma=0.012$.}
    \label{fig:noised_100_vs_1250_training}
  \end{center}
\end{minipage}
\end{figure}
\vspace{-0.2in}

\begin{figure}[H]
\begin{center}
  \includegraphics[width=2.7in]{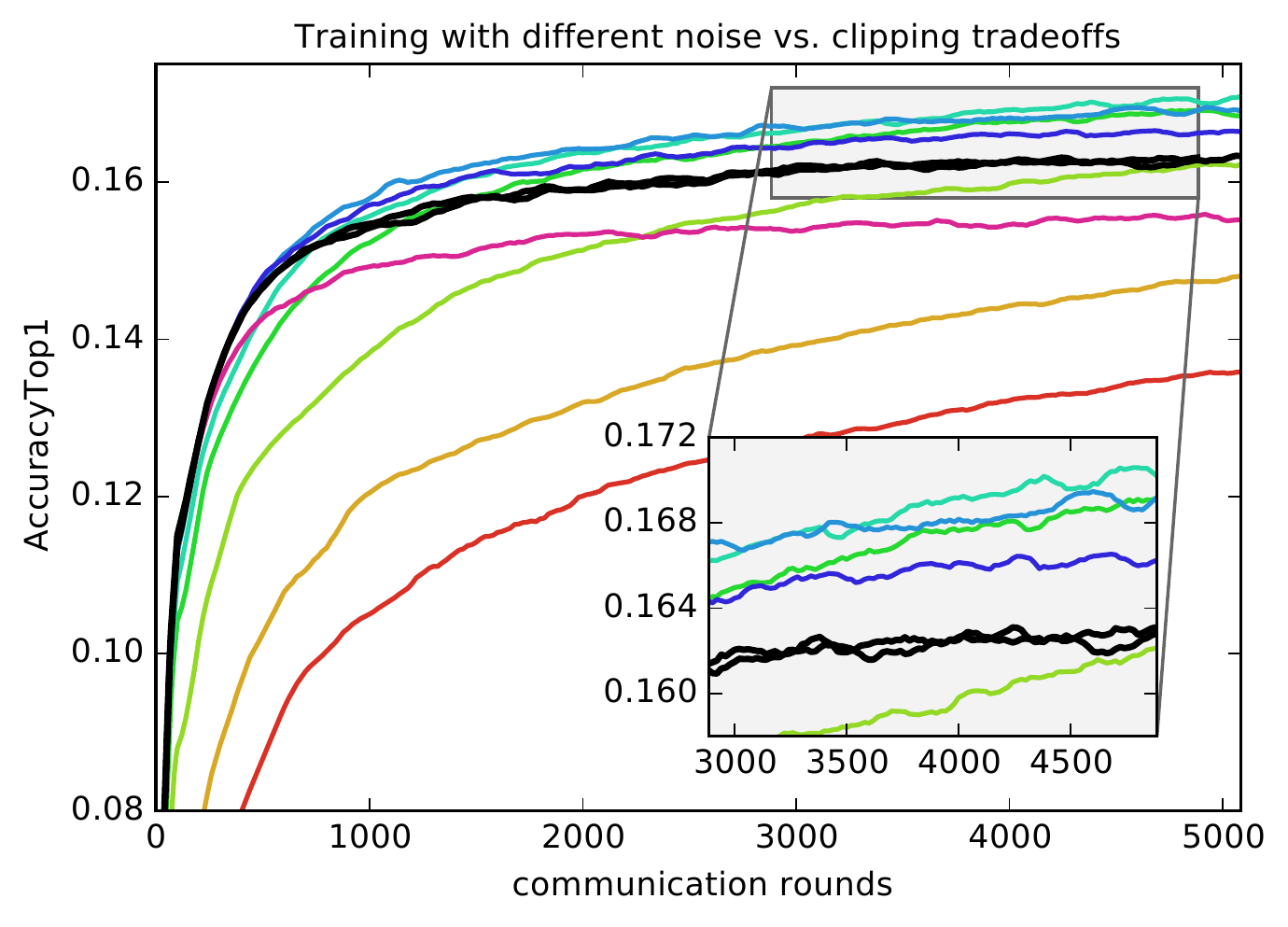}
  \includegraphics[width=2.7in]{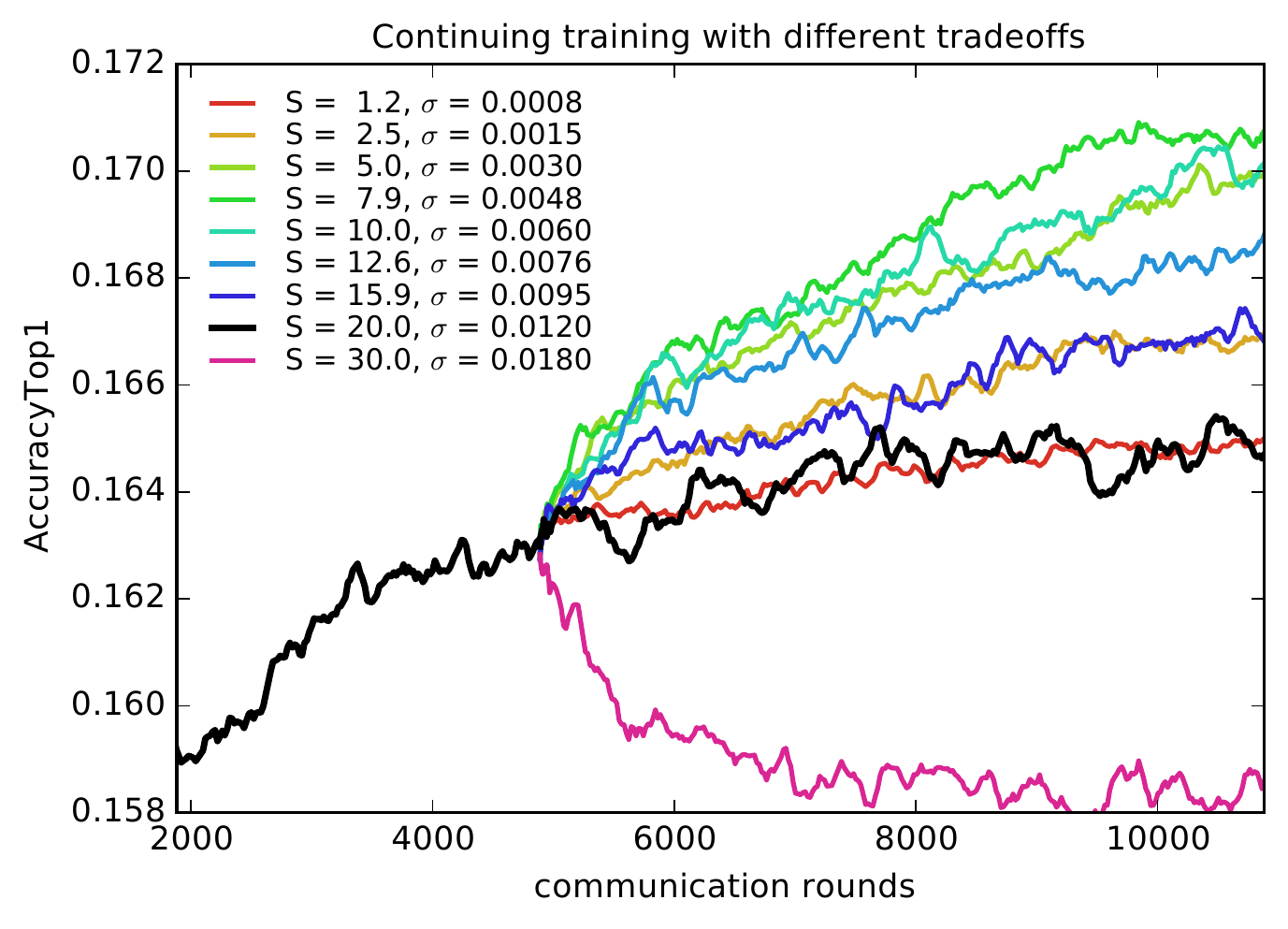}
  \mycaptionof{figure}{The effect of different noise vs. clipping
    tradeoffs on convergence. Both plots use the same legend, where we
    vary $S$ and $\sigma$ together to maintain the same $\noisescale =
    0.06$ with 100 users (actually used), or $z=1$ with 1667 users. We
    take $S=20$ and $\sigma=0.012$ (black line) as a baseline; the
    left-hand plot shows training from a randomly initialized model,
    and includes two different runs with $S=20$, showing only mild
    variability. For the right-hand plot, we took a snapshot of the
    $S=20$ model after 4885 initial rounds of training, and resumed
    training with different
    tradeoffs.
}
   \label{fig:continued_training}
\end{center}
\end{figure}

\clearpage
\section{Additional Experiments}
\label{sec:additional}

\paragraph{Experiments with SGD}
We ran experiments using \fedsgd taking $\lbs=1600$, that is,
computing the gradient on each user's full local dataset. To allow
more iterations, we used $\eC=50$ rather than 100.
Examining Figures~\ref{fig:effect_of_clipping_sgd} and
\ref{fig:effect_of_noise_sgd}, we see $S=2$ and $\sigma=2\cdot10^{-3}$
are reasonable values, which suggests for private training we would
need in expectation $qW = S/\sigma = 1500$ users per round, whereas
for \fedavg we might choose $S=15$ and $\sigma=10^{-2}$ for $\eC
=qW=1000$ users per round. That is, the relative effect of the ratio
of the clipping level to noise is similar between \fedavg and
\fedsgd. However, \fedsgd takes a significantly larger number of
iterations to reach equivalent accuracy. Fixing $z=1$, $\eC = 5000$
(the value that produced the best accuracy for a private model in
Table~\ref{table:dpaccuracy}) and total of 763,430 users gives $(3.81,
10^{-9})$-DP after 3000 rounds and $(8.92, 10^{-9})$-DP after 20000
rounds, so there is indeed a significant cost in privacy to these
additional iterations.

\setlength{\pw}{2.4in}
\begin{figure}[t]
\begin{minipage}{.49\textwidth}
   \begin{center}
   \includegraphics[width=\pw]{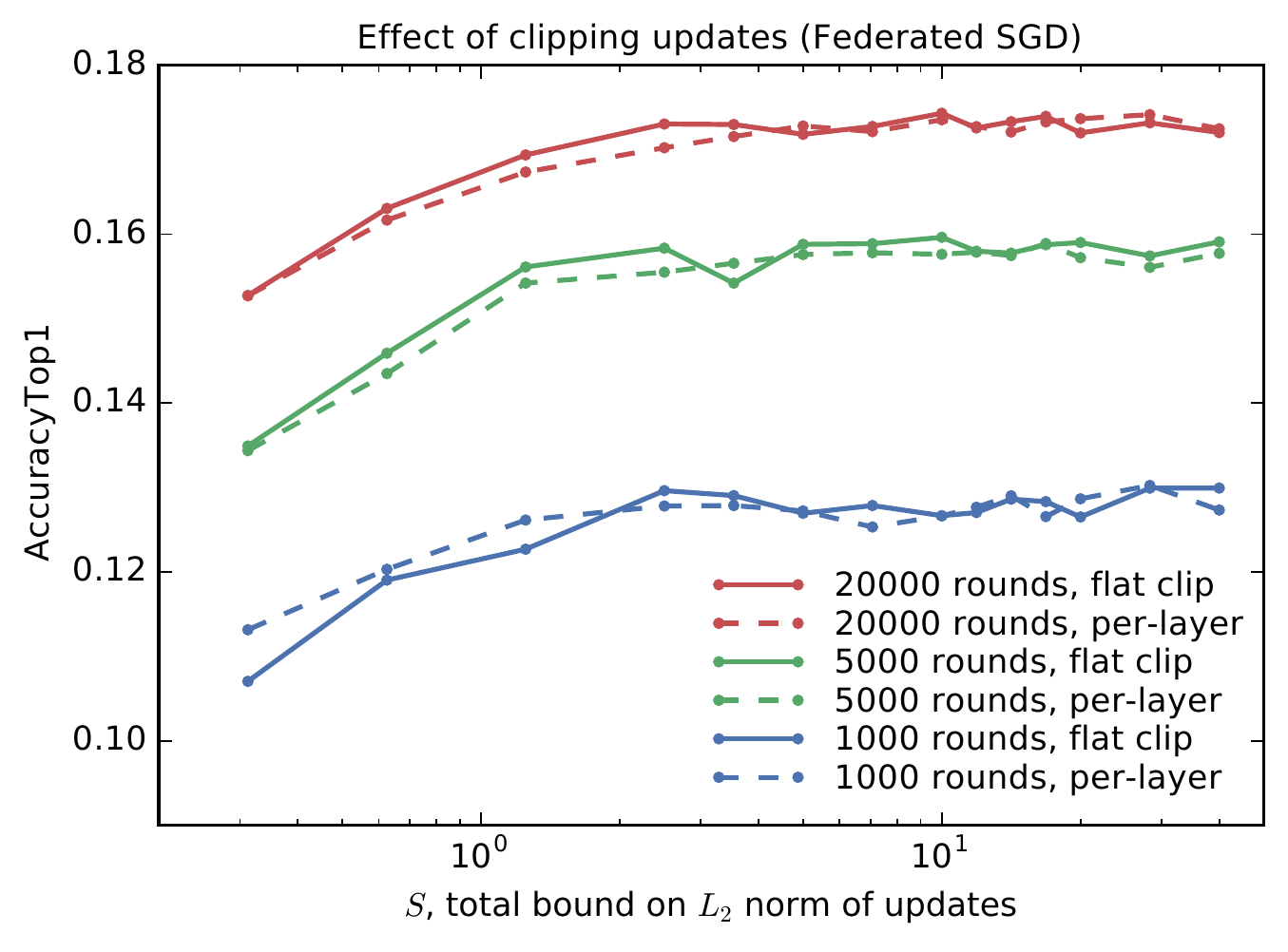} \\
   \mycaptionof{figure}{Effect of clipping on \fedsgd with $\eC = 50$ users
   per round and a learning rate of $\eta=6$. A much smaller clipping level $S$
   can be used compared to \fedavg.}
   \label{fig:effect_of_clipping_sgd}
  \end{center}
\end{minipage}%
\hfill
\begin{minipage}{.49\textwidth}
  \begin{center}
   \includegraphics[width=\pw]{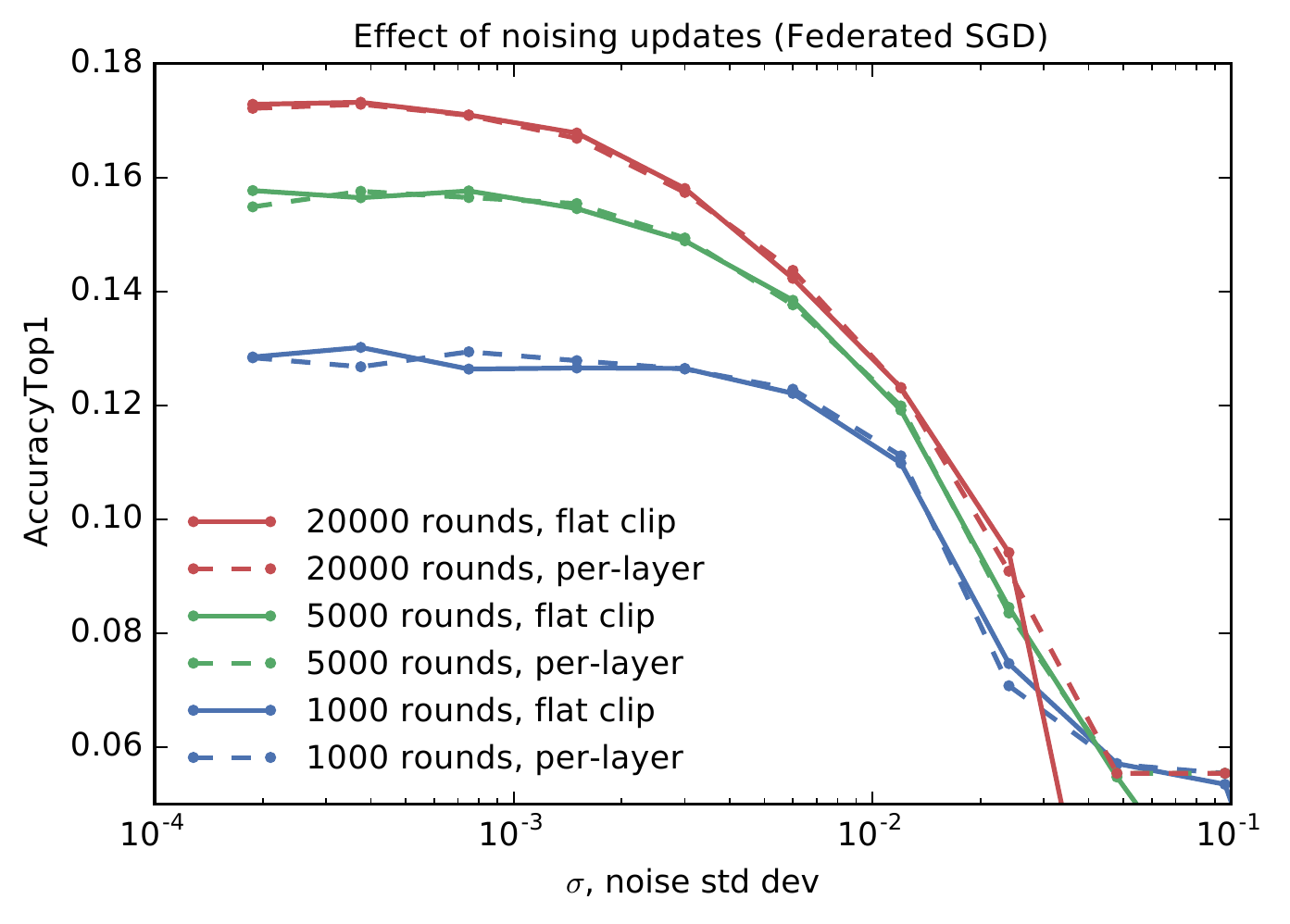} \\
   \mycaptionof{figure}{Effect of noised updates on \fedsgd with $S=20$ (based
   on Figure~\ref{fig:effect_of_clipping_sgd}, a smaller value would actually
   be better when doing private training). \fedsgd is more sensitive to
   noise than \fedavg, likely because the updates are smaller in magnitude.}
   \label{fig:effect_of_noise_sgd}
  \end{center}
\end{minipage}
\end{figure}

\paragraph{Models with larger dictionaries}
We repeated experiments on the impact of clipping and noise on models
with 20000 and 30000 token dictionaries, again using \fedavg training
with $\eta=6$, equally weighted users with 1600 tokens, and
$\eC=100$ expected users per round. The larger
dictionaries give only a modest improvement in accuracy, and do not
require changing the clipping and noise parameters despite having
significantly more parameters. Results are given in
Figures~\ref{fig:effect_of_clipping_20k_30k} and
\ref{fig:noise_20k_30k}.

\begin{figure}[t]
\begin{minipage}{.49\textwidth}
   \begin{center}
   \includegraphics[width=\pw]{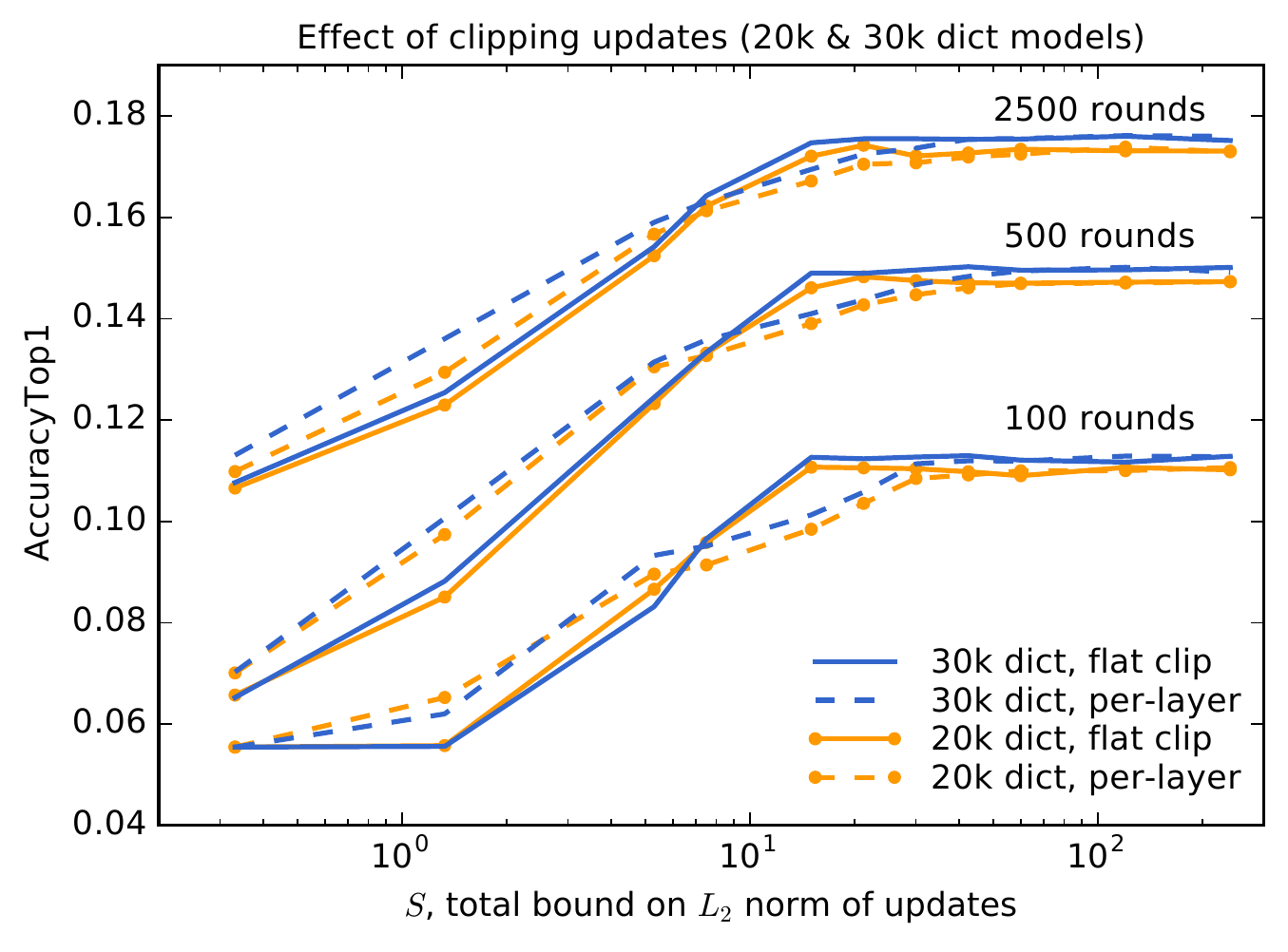} \\
   \mycaptionof{figure}{Effect of clipping on models with larger dictionaries (20000 and 30000 tokens).}
   \label{fig:effect_of_clipping_20k_30k}
  \end{center}
\end{minipage}%
\hfill
\begin{minipage}{.49\textwidth}
  \begin{center}
   \includegraphics[width=\pw]{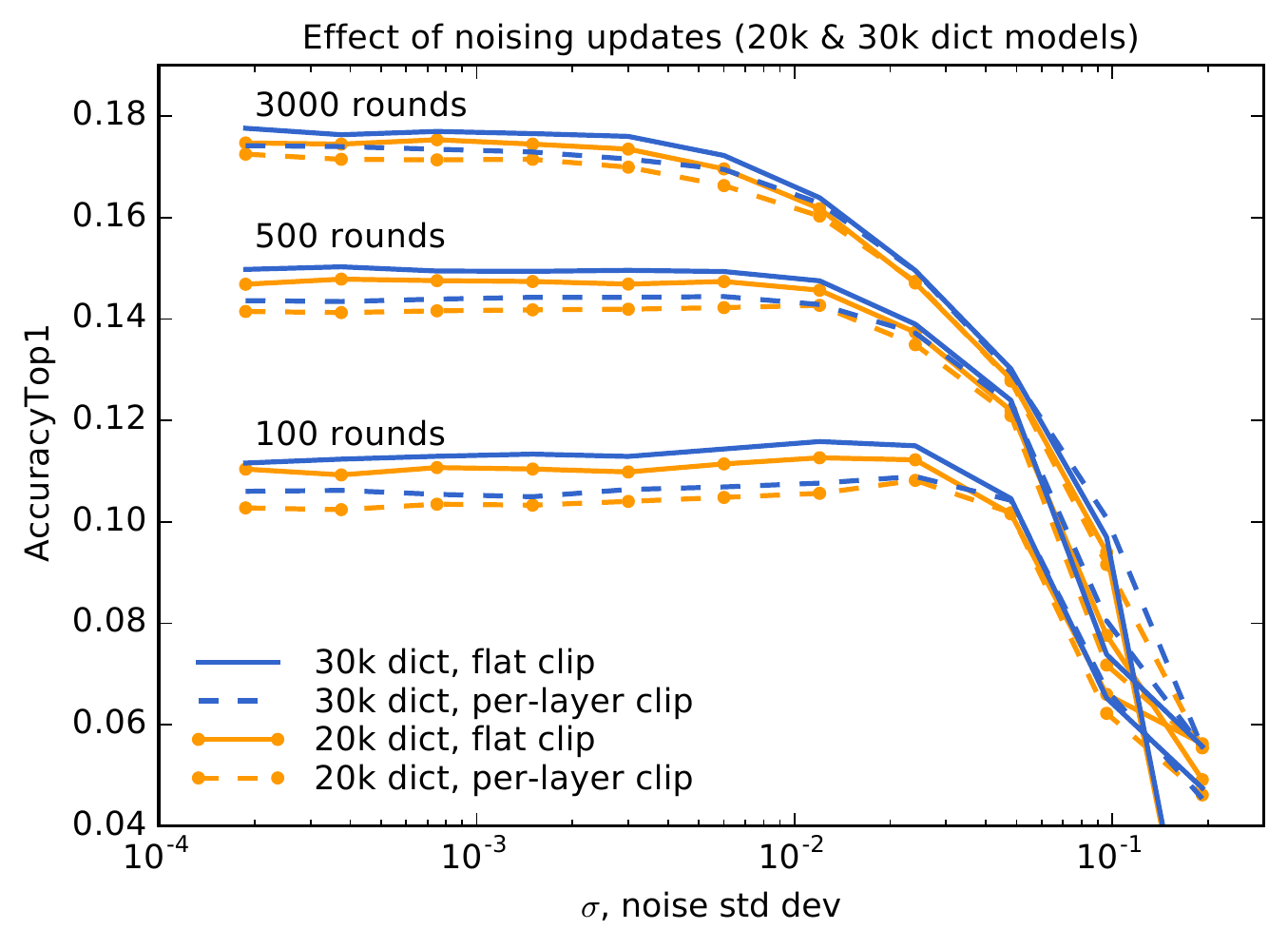} \\
   \mycaptionof{figure}{Effect of noised updates on models with larger dictionaries, when clipped at $S=20$.}
   \label{fig:noise_20k_30k}
  \end{center}
\end{minipage}
\end{figure}

\paragraph{Other experiments}
We experimented with adding an explicit $\ltwo$ penalty on the model
updates (not the full model) on each user, hoping this would
decrease the need for clipping by preferring updates with a smaller
$\ltwo$ norm. However, we saw no positive effect from this.